\documentclass{scrartcl}
\usepackage{siunitx}
\usepackage[T1]{fontenc}
\usepackage[english]{babel}
\usepackage{mathtools,mathrsfs}
\usepackage{tabularx}
\usepackage{amsmath,amssymb,amsthm}
\usepackage{algorithm}
\usepackage[noend]{algpseudocode}
\usepackage{listings, enumitem}
\usepackage[nohints]{minitoc}
\usepackage{graphicx}
\usepackage{subcaption}
\usepackage{dsfont}
\usepackage{tikz}
\usepackage{esint}
\usepackage[nomessages]{fp}
\usepackage{hyperref}

\newcommand{\dx}{\,\mathrm{d}}

\DeclareMathOperator{\R}{{\mathbb{R}}}

\renewcommand{\vec}[1]{\mathbf{#1}}

\newcommand{\X}{\mathbb{X}}

\renewcommand{\S}{\ensuremath{\mathbb{S}}}

\DeclareMathOperator*{\argmin}{arg\,min} %handles subscripts like \lim
\DeclareMathOperator*{\supp}{supp} %handles subscripts like \lim
\DeclareMathOperator{\TV}{TV}
\DeclareMathOperator{\KL}{KL}

\renewcommand{\vec}[1]{\mathbf{#1}}

\newcommand{\N}{\mathbb{N}}
\newcommand{\Id}{\mathrm{Id}}
\newcommand{\Hyp}{\mathrm{Net}_\phi(\R^d)}

\DeclareMathOperator{\Lip}{Lip}

\newtheorem{theorem}{Theorem}[section]
\newtheorem{lemma}[theorem]{Lemma}
\newtheorem{remark}[theorem]{Remark}

\newtheorem{proposition}[theorem]{Proposition}

\title{On the Effect of Initialization: The Scaling Path of 2-Layer Neural Networks}
\date{\today}
\author{Sebastian Neumayer, Lénaïc Chizat, Michael Unser}

\begin{document}

\maketitle
\begin{abstract}
	In supervised learning, the regularization path is sometimes used as a convenient theoretical proxy for the optimization path of gradient descent initialized from zero.
	In this paper, we study a modification of the regularization path for infinite-width 2-layer ReLU neural networks with nonzero initial distribution of the weights at different scales.
	By exploiting a link with unbalanced optimal-transport theory, we show that, despite the non-convexity of the 2-layer network training, this problem admits an infinite-dimensional convex counterpart.
	We formulate the corresponding functional-optimization problem and investigate its main properties.
	In particular, we show that, as the scale of the initialization ranges between $0$ and $+\infty$, the associated path interpolates continuously between the so-called kernel and rich regimes.
	Numerical experiments confirm that, in our setting, the scaling path and the final states of the optimization path behave similarly, even beyond these extreme points.
\end{abstract}

\section{Introduction}\label{sec:Intro}

The mathematical theory of artificial neural networks (NNs) can be tackled from either a  dynamic or a static viewpoint\footnote{\href{https://mjt.cs.illinois.edu/dlt/}{Matus Telgarsky, Deep learning theory lecture notes: https://mjt.cs.illinois.edu/dlt/}}.
In the dynamic approach, one considers a NN in combination with a training algorithm.
Then, one studies the statistical properties of the NN along, and at the end, of the training cycle.
In the static approach, one studies NNs as a statistical hypothesis (or candidate) space, independently of any training routine.
This space is typically endowed with a norm (or, more generally, a metric) in parameter space, which acts as a regularizer (measure of complexity).
Both approaches address distinct aspects.
The dynamic approach studies the objects that are the most relevant to practice, but faces the difficulty that those are less tractable theoretically.
Thus, much fewer statistical results are known compared to static approaches.

Let us consider a parametric model $\phi\colon\mathbb{R}^p\to \mathcal{F}$, where $\mathbb{R}^p$ is the space of parameters and $\mathcal{F}$ a space of functions, and let $\mathcal L\colon\mathcal{F}\to \mathbb{R}$ be an objective function such as the empirical risk.
In the dynamic approach, a convenient object of study is the \emph{optimization path} that results from a gradient flow.
This path $(\boldsymbol \theta_t)_{t\geq 0}$ starts from a given initialization $\boldsymbol\theta_0\in \mathbb{R}^p$ and solves
\begin{align}\label{eq:GF}
	\frac{\dx}{\dx t} \boldsymbol \theta_t = -\nabla \mathcal L(\phi(\boldsymbol \theta_t)),
\end{align}
as well as the associated path $ \phi(\boldsymbol \theta_t)$ in function space.
Many refinements are of course possible to make the model more realistic such as taking into account stochasticity~\cite{li2019stochastic, pesme2021implicit}, large stepsizes~\cite{wang2022large}, or momentum~\cite{su2014differential}.
As for static analyses, they often focus on the constrained path $\boldsymbol \theta^*_\delta = \argmin_{\Vert \boldsymbol \theta\Vert\leq \delta} \mathcal L(\phi(\boldsymbol \theta))$ or the \emph{regularization path}
\begin{align}\label{eq:regularization-path}
	\boldsymbol \theta^*_\lambda = \argmin_{\boldsymbol \theta} \mathcal L(\phi(\boldsymbol \theta)) + \lambda \Vert \boldsymbol \theta\Vert_2^2.
\end{align}
In the simple context of linear parameterizations, formalized as $\phi(\boldsymbol \theta)=\boldsymbol \theta^\top \Phi$ for some $\Phi\in\mathcal{F}^p$ with the initial parameter $\boldsymbol \theta_0=\mathbf{0}$ for~\eqref{eq:GF}, the two approaches are tightly interconnected.
More precisely, one creates a close link between the optimization (dynamic) and the regularization (static) paths~\cite{suggala2018connecting, ali2020implicit} by letting the tuning parameter take the form $\lambda=1/(2t)$.

\paragraph{Scaling Path} It is perhaps too optimistic to expect such a tight connection for nonlinear NNs.
Indeed, this connection breaks, for example, in the cases studied in \cite{razin2020implicit, woodworth2020kernel}.
Still, if the regularization path was to preserve some of the key characteristics of optimization paths (such as certain asymptotic behaviors) this would make the static approach relevant to a better understanding of practical NNs.

For the rest of this work, we notate $\mathcal L$ as the empirical risk associated with samples $(\vec x_k,y_k) \in \R^d \times \R$, $k=1,\ldots,n$ and some loss function $L\colon \R^d \times \R \to [0,+\infty]$.
To include the case of arbitrarily wide NNs, we replace the parameter space $\R^p$ with $\mathcal{P}_2 (\R^p)$.
Accordingly, we shall denote the parametrization function of our regression problem by $\phi[\mu] = \int_{\R^p} \phi(\vec w) \dx \mu(\vec w)$.

A serious obstacle to the establishment of a link between~\eqref{eq:GF} and~\eqref{eq:regularization-path} in the case of NNs is that it is inconvenient to initialize the optimization from $\mu_0=\delta_0$ since this is often a stationary point of~\eqref{eq:GF}.
As a remedy, one may instead initialize from the uniform distribution $\mu_0$ on the sphere with radius $\alpha$.
Hence, we study a modification of the regularization path~\eqref{eq:regularization-path} that takes this nonzero initialization and the generalization to measures into account.
More precisely, for fixed $\mu_0$ and $\lambda>0$, we define the \emph{scaling path} with scale $\alpha$ as 
\begin{align}\label{eq:ProxPath}
	\mu^*_{\alpha} = \argmin_{\mu \in \mathcal{P}_2 (\R^p)} \sum_{k=1}^n L\bigl(\phi[\mu](\vec x_k), y_k \bigr) +\lambda R_\alpha(\mu, \mu_0),
\end{align}
where the functional $R_\alpha(\mu, \mu_0) = (1+ \alpha^2) W_2(\mu, {S_\alpha}_\#\mu_0)^2$ with $S_\alpha(\vec w) = \alpha \vec w$ acts as our scale-dependent regularizer.
%A similar object has been previously studied in~\cite{woodworth2020kernel} for linear predictors parameterized as (differences of) squares of parameters.
A first important observation is that 2-layer ReLU NNs are covered within our framework if we choose
\begin{equation}
	\phi_{\mathrm{ReLU}}(\vec w, \vec x) = w_1 (\vec w_2^T \vec x)^+, \qquad \vec w = (w_1, \vec w_2) \in \R \times \R^{d}.
\end{equation}
Further, the problem in \eqref{eq:ProxPath} is convex, which allows the use of standard optimization tools.
In the context of 2-layer ReLU NNs with vanishing initialization scale $\alpha$, we shall see in Section~\ref{sec:Interp} that \eqref{eq:ProxPath} reduces (up to rescaling of $\lambda$) to the regularization path \eqref{eq:regularization-path}.

For our analysis, it is actually more convenient to study \eqref{eq:ProxPath} from a functional perspective by considering the resulting function $f = \phi[\mu]\colon \R^d \to \R$.
As different $\mu$ can lead to the same network $f$, the regularizer $R_\alpha(\cdot,\mu_0)$ needs to be replaced by a complexity measure $N_\alpha(f,\mu_0)$ that takes the whole equivalence class into account.
Essentially, $N_\alpha(f,\mu_0)$ describes the distance of a particular NN $f$ to some initialization $\mu_0$.
With this notation, the scaling path \eqref{eq:ProxPath} now takes the form
\begin{equation}
	\argmin_{f \in \Hyp} \sum_{k=1}^n L\bigl(f(\vec x_k),y_k\bigr) + \lambda N_\alpha(f,\mu_0),
\end{equation}
where $\Hyp$ is a suitable space of functions.
Our study of $N_\alpha(f,\mu_0)$ will reveal an interesting link between \eqref{eq:ProxPath} and the theory of unbalanced optimal transport.
Based on this link, we can prove our main result, namely, that the scaling path is $\Gamma$-convergent in some suitably chosen metric space.
In particular, this implies that the family of minimizing networks $f^*_{\alpha} \in C(\R^d)$ depends continuously on $\alpha$.
For the limiting cases of vanishing and infinite scales $\alpha$ in~\eqref{eq:ProxPath}, we get convergence to two well-known settings from the literature, which are discussed in the next paragraph.

\paragraph{Limits of the Scaling Path}
As the scale $\alpha$ in~\eqref{eq:ProxPath} vanishes (${S_\alpha}_\#\mu_0 \to \delta_0$), our problem turns into $\ell^2$-weight regularization with
\begin{equation}\label{eq:Rich}
	\argmin_{\mu \in \mathcal{P}_2 (\R^p)} \sum_{k=1}^n L\bigl(\phi[\mu](\vec x_k), y_k \bigr) + \lambda\int_{\R^p} \Vert \vec w \Vert^2 \dx  \mu(\vec w).
\end{equation}
This problem was thoroughly investigated \cite{Bach2017,Ongie2020b,NeuUns2022}, and is known to admit sparse solutions; namely, minimizers $\hat \mu$ that consist of few atoms $\delta_{\vec w_k}$.
In the case of ReLU networks with $\phi_{\mathrm{ReLU}}$, these correspond to functions of the form
$f_\mu = \sum_{k=1}^n w_{k,1} (\vec w_{k,2}^T \vec x)^+$.
Further, \eqref{eq:Rich} leads to predictors with strong statistical properties, such as good adaptivity to anisotropic target functions.
Following~\cite{woodworth2020kernel}, we refer to this formulation as the ``rich regime''.
This formulation is known to capture end-of-training behavior of the gradient flow of 2-layer NNs in certain contexts, such as with the logistic loss~\cite{CB2020, lyu2021gradient} or with a small initialization~\cite{boursier2022gradient}.

In contexts such as large initialization with square loss, the training of NNs behaves instead according to the neural tangent kernel (NTK) theory~\cite{JacGab2018, arora2019fine, BieMai2019}.
There, the kernel in general form is given by
\begin{align}
	K(\vec x, \vec x^\prime) =& \int_{\R^p}\bigl(\nabla_{\vec w} \phi(\vec w, \vec x)\bigr)^T \nabla_{\vec w} \phi(\vec w, \vec x^\prime) \dx  \mu_0(\vec w)
\end{align}
and depends on the initial distribution $\mu_0 \in \mathcal P_2(\R^{p})$.
%Here, the first part corresponds to the so-called random features, and the second one to the neural tangent class \cite{GhoMei2021}.
In this \emph{kernel} (a.k.a.~\emph{lazy}) regime, the gradient flow in the large-time limit solves the associated kernel-ridge regression problem
\begin{equation}\label{eq:NTK}
	\argmin_{f \in \mathcal H_K} \sum_{k=1}^n L(f(\vec x_k),y_k) + \lambda \Vert f \Vert_{\mathcal H_K}^2,
\end{equation}
which we identify as the limit of \eqref{eq:ProxPath} as $\alpha \to \infty$.
Note that the solution of this problem can be written in the form $\hat f(\vec x) = \sum_{k=1}^n K(\vec x,\vec x_k)c_i$ with $c_i \in \R$.
Equivalently, we can also investigate the problem
\begin{equation}\label{eq:neural_tangent_kernel}
	\argmin_{T \in L^2(\R^{p}, \mu_0)} \sum_{k=1}^n L\bigl(\bigl\langle T, \nabla_{\vec w} \phi(\cdot, \vec x_k) \bigr\rangle_{L^2(\R^{p}, \mu_0)},y_k\bigr) + \lambda \Vert T \Vert_{L^2(\R^{p},\mu_0)}^2
\end{equation}
associated to the corresponding feature map (see \cite{Berlinet2004} for details), which leads to the same solution in function space.
The underlying feature map $\nabla_{\vec w} \phi(\cdot, \vec x)$ is related to the Taylor expansion of the NN parameterization function $\phi$ around the initial parameters.

\paragraph{Outline}
To study the scaling path~\eqref{eq:ProxPath}, we introduce and analyze the complexity measure $N(f,\mu_0)$ in Section~\ref{secInfWideNetworks}.
Based on the developed theory for $N(f,\mu_0)$, we investigate in Section~\ref{sec:Interp} the associated family of functional-optimization problems \eqref{eq:ProxPath}.
As our main result, we prove that the underlying family of functionals is $\Gamma$-convergent and that the rich regime~\eqref{eq:Rich} and the kernel regime~\eqref{eq:NTK} are the limits for $\alpha \to 0$ and $\alpha \to \infty$, respectively.
Our theoretical results are illustrated in Section~\ref{Sec:NumExample} by a simulation in which we compare the scaling path and the final state of the optimization path for several scales $\alpha$.
Finally, we draw conclusions in Section~\ref{sec:Conclusions}.

\section{Infinite-Width Neural Networks}\label{secInfWideNetworks}
\begin{table}[t]
	\centering
	\footnotesize
	\begin{tabular}{l|l}
		$\mathcal{P}_2 (\R^p)$ & probability measures with finite second moments (with 2-Wasserstein metric $W_2$)\\
		$\mathcal M^+(\S^{p-1})$ & positive measures on the sphere $\S^{p-1}$\\
		$\mathcal H_K$ & reproducing kernel Hilbert space corresponding to $K$\\
		$C(\X)$ & continuous functions on metric space $\X$ (equipped with supremum norm)\\
		$L^2(\R^d, \mu_0)$ & square integrable functions with norm weighted by $\mu_0$\\ 
		$W_{\text{loc}}^{1,\infty}(\R^p)$ & weakly differentiable functions with finite $W^{1,\infty}$-norm on compact sets\\ 
	\end{tabular}
	\caption{Function and measure spaces.}
	\label{tab:notation}
\end{table}
The function and measure spaces used throughout this manuscript are briefly introduced in Table~\ref{tab:notation}.
In abstract form, we can parameterize infinite-width NNs using probability measures with finite second-order moments $\mu \in \mathcal{P}_2 (\R^p)$, where $\R^p$ is the parameter space, and a function $\phi \colon \R^p \times \R^d \to \R$ that, in our case, satisfies the following properties.
\begin{itemize}
	\item \emph{2-homogeneity in $w$}: $\phi(r\vec w, \vec x) = r^2\phi(\vec w, \vec x)$ for all $(r,\vec w,\vec x) \in {\R^+} \times \R^px \times \R^d$.
	%\item \emph{balancedness}: there is a continuous map $T \colon \S^{p-1} \to \S^{p-1}$ such that for all $ \theta \in \S^{p-1}$ it holds $\phi(T(\theta),\cdot) = -\phi(\theta,\cdot)$.
	\item \emph{Regularity in $w$}: for every $\vec x \in \R^d$, $\phi(\cdot,\vec x) \in W_{\text{loc}}^{1,\infty}(\R^p)$ is twice continuously differentiable on an open cone $C_{\vec x} \subset \R^p$ with full Lebesgue measure $\lambda$, so that $\lambda(\R^p \setminus C_{\vec x})=0$, and $\Vert \nabla_{\vec w} ^2 \phi(\cdot,\vec x)\Vert$ is uniformly bounded on $C_{\vec x}$.
	\item \emph{Lipschitz regularity in $\vec x$}: For every $\vec w \in \R^p$, $\vec x \mapsto \phi(\vec w,\vec x)$ is Lipschitz-continuous, and there exists a constant $C>0$ such that $\Lip(\phi(\vec w,\cdot)) \leq C\Vert\vec  w \Vert^2$.
\end{itemize}
Some of these conditions are similar to those of \cite{CB2020}.
The first assumption implies that $\int_{\R^p} \phi(\vec w,\vec x) \dx \mu(\vec w) < \infty$ for all $\mu \in \mathcal{P}_2 (\R^p)$ and $\vec x \in \R^d$.
The first two assumptions together imply that $\nabla_{\vec w}  \phi(\cdot,\vec x) $ is (positively) 1-homogeneous in its first variable on $C_{\vec x}$, in the sense that $\nabla_{\vec w} \phi(r\vec w, \vec x) = r\nabla_{\vec w} \phi(\vec w, \vec x)$ for all $(r,\vec w) \in {\R^+} \times C_{\vec x}$.
Hence, $\phi(\cdot,\vec x) \in W_{\text{loc}}^{1,\infty}(\R^p)$ also implies that $\nabla_{\vec w}  \phi(\cdot,\vec x) \in L^2(\R^p,\mu)$ for all $\mu \in \mathcal{P}_2 (\R^p)$ and $\vec x \in \R^d$.
Using the function $\phi$, we define an associated space of infinite width NNs as
\begin{equation}
	\Hyp = \biggl\{\int_{\R^p} \phi(\vec w,\cdot) \dx \mu(\vec w): \mu \in \mathcal{P}_2 (\R^p)\biggr\}.
\end{equation}
Note that the parameterization $\mu$ of a NN $f \in \Hyp$ is not necessarily unique.
\begin{remark}\label{rem:ReLU}
	One can readily verify that $2$-layer ReLU NNs with parameterization function $\phi_{\mathrm{ReLU}}(\vec w, \vec x) = w_1 (\vec w_2^T \vec x)^+$, $\vec w = (w_1, \vec w_2) \in \R^{d+1}$, fit into this abstract framework.
	Here, $w_1$ parameterizes the scalar output weights and  $\vec w_2$ parameterizes the hidden layer.
	To allow for bias vectors, we can pad the input vector $\vec x$ with a $1$ at the end and treat the biases as part of the weights $\vec w$.
	Given an atomic measure $\mu = \sum_{k=1}^n \delta_{\vec w_k}$, the associated finite width NN reads $f_\mu = \sum_{k=1}^n w_{k,1} (\vec w_{k,2}^T \vec x)^+$.
	Finally, using the characteristic function $\chi_{\R^+}$ of the positive reals, the NTK of a 2-layer ReLU NN with initialization $\mu_0$ is given by
	\begin{equation}
		K_{\mathrm{ReLU}}(\vec x, \vec x^\prime)=\int_{\R^{d+1}} \Big((\vec w_2^T \vec x)^+ (\vec w_2^T \vec x^\prime)^+ +  w_1^2 \vec x^T \vec x^\prime \chi_{\R^+}(\vec w_2^T \vec x) \chi_{\R^+}(\vec w_2^T \vec x^\prime)\Big) \dx \mu_0(\vec w).
	\end{equation}
\end{remark}
Although $2$-layer ReLU NNs are the most relevant choice of $\phi$ from a practical viewpoint, we prefer to carry out our analysis for this general class of functions $\phi$.
For any $\vec x,\vec y \in \R^d$ and $f = \int_{\R^p} \phi(\vec w,\cdot) \dx \mu(\vec w) \in \Hyp$, it holds that
\begin{equation}
	\vert f(\vec x) - f(\vec y) \vert \leq C \int_{\R^p} \Vert \vec w \Vert^2 \dx \mu(\vec w) \Vert \vec x - \vec y\Vert.
\end{equation}
Hence, all functions in $\Hyp$ are Lipschitz-continuous.
Therefore, $\Hyp$ is a subset of $C(\R^d$).
In Section \ref{sec:ComMeasure}, we construct a complexity measure that encodes the distance of a given NN to a reference parameterization $\mu_0 \in \mathcal P_2(\R^{p})$, which could be, for example, the initialization of the NN before it is trained according to the gradient flow~\eqref{eq:GF}.

\subsection{Measure of Complexity for Neural Networks}\label{sec:ComMeasure}
In the following, we rely heavily on optimal transport and, in particular, on the 2-Wasserstein metric $W_2$ \cite{AGS2005,Villani2009}.
Let $\mu_0 \in \mathcal P_2(\R^{p})$ be a probability measure with polar disintegration $\mu_0(\!\dx r, \!\dx \boldsymbol \theta)=\mu_0(\!\dx r| \boldsymbol \theta) \hat \mu_0 (\!\dx \boldsymbol \theta)$, where  $\hat \mu_0 \in \mathcal M^+(\S^{p-1})$ and $\int_{{\R^+}} r^2 \mu_0(\!\dx r|\boldsymbol \theta) = 1$ for $\hat \mu_0$-a.e.\ $\boldsymbol \theta \in \S^{p-1}$.
We then define the \emph{complexity measure} $N(\cdot,\mu_0) \colon \Hyp \to {\R^+}$ as
\begin{align}
	N(f,\mu_0) &= \inf_{\mu \in \mathcal{P}_2 (\R^p)} \biggl\{W_2^2(\mu,\mu_0): f = \int_{\R^p} \phi(\vec w,\cdot) \dx \mu(\vec w) \biggr\}\label{eq:N_relax}.
\end{align}
Loosely speaking, $N(f,\mu_0)$ encodes by how much the parameter $\mu$ needs to move away from a reference measure $\mu_0$ in order to realize the NN $f$.
Using the Monge formulation of optimal transport, we obtain the upper bound
\begin{equation}
	N(f,\mu_0) \leq \inf_{T \in L^2(\R^p, \mu_0)} \biggl\{ \Vert T - \Id \Vert_{L^2(\mu_0)}^2 :  f = \int_{\R^p} \phi(\vec w,\cdot) \dx [T_\#\mu_0](\vec w)\biggr\}\label{eq:N},
\end{equation}
where $T_\#\mu_0 = \mu_0(T^{-1}(\cdot))$ denotes the push-forward measure of $\mu$ under $T$.
Because optimal transport maps do not necessarily exist, the right-hand side of~\eqref{eq:N} is indeed an infimum rather than a minimum.
However, the equality of \eqref{eq:N_relax} and \eqref{eq:N} holds when $\mu_0$ is absolutely continuous with respect to the Lebesgue measure $\lambda$.
This relation turns out to be useful for the derivation of our main result in Section~\ref{sec:Interp}.
\begin{remark}
	Recently, the idea of using an optimal-transport-based complexity measure for NNs has also been pursued in \cite[Section 5]{CheVanBru2022}.
	In their simplest instance, where the underlying function space is isomorphic to $\R^d$ and $\sigma_2$ is the ReLU, they investigate the same NNs as we in Remark \ref{rem:ReLU}.
	Albeit closely related, their complexity measure $\gamma_2^\dagger$ differs from ours since transport plans $\pi$ are supported on $\R \times \R^d \times \R^{d}$ instead of $\R^{d+1} \times \R^{d+1}$ and the transport cost is only computed with respect to $\vec w_2 \in \R^{d}$.
	Based on this choice, they are able to derive Rademacher complexity bounds for $\gamma_2^\dagger$, which lead to a posteriori generalization error bounds for gradient-desecent-trained NNs depending on a notion of the length of the optimization path \eqref{eq:GF}.
	For our more specific 2-homogeneous setting, the aim is instead to study fine properties of the scaling path \eqref{eq:ProxPath} as $\alpha$ varies.
\end{remark}

\subsection{Properties of the Complexity Measure}
First, we show that the complexity measure $N$ satisfies a homogeneity property.
\begin{lemma}[Homogeneity]
	For all $f \in \Hyp$, $\alpha>0$, and $\mu_0 \in \mathcal{P}_2 (\R^p)$, it holds with $S_\alpha\colon \R^p \to \R^p$ given by $\vec w \mapsto \alpha \vec w$ that
	\begin{equation}
		\alpha^2 N(f,\mu_0) = N(\alpha^2 f, {S_\alpha}_\# \mu_0).
	\end{equation}
\end{lemma}
\begin{proof}
	Since $f = \int_{\R^p} \phi(\vec w,\cdot) \dx \mu(\vec w)$ and $\phi$  is 2-homogeneous, we have that
	\begin{equation}
		\int_{\R^p} \phi(\vec w,\cdot) \dx [{S_\alpha}_\# \mu](\vec w) = \int_{\R^p} \phi(\alpha \vec w,\cdot) \dx \mu(\vec w) = \alpha^2f.
	\end{equation}
	The result then follows from $W_2({S_\alpha}_\# \mu, {S_\alpha}_\# \mu_0)^2 = \alpha^2 W_2(\mu,\mu_0)^2$.
\end{proof}
The constraint in~\eqref{eq:N_relax} has a very specific structure.
Using the 2-homogeneous projection operator $\Pi_2 \colon \mathcal{P}_2 (\R^p) \to   \mathcal{M}^+ (\S^{p-1})$ characterized by
\begin{equation}
	\int_{\S^{p-1}}
	\varphi(\boldsymbol \theta) \dx[\Pi_2(\mu)]\boldsymbol (\theta) = \int_{\R^p} \Vert \vec w \Vert^2 \varphi(\vec w/\Vert \vec w \Vert) \dx \mu (\vec w)
\end{equation}
for any $\varphi \in C(\S^{p-1})$, we rewrite~\eqref{eq:N_relax} as
\begin{align}\label{eq:equiv_constr}
	f = \int_{\R^p} \phi(\vec w,\cdot) \dx \mu(\vec w) = \int_{\S^{p-1}} \phi(\boldsymbol \theta,\cdot) \dx [\Pi_2(\mu)](\boldsymbol \theta).
\end{align}
Based on $\Pi_2$, we are now in the position to introduce the distance \smash{$\widehat W_2$} on $\mathcal{M}^+(\mathbb{S}^{p-1})$ known as the  Hellinger--Kantorovich or the Wasserstein--Fisher--Rao distance~\cite{LMS2015,kondratyev2016new, chizat2018interpolating}. We consider the formulation introduced by~\cite{LMS2015}, which is given for $\hat \mu_1,\hat \mu_2 \in \mathcal M^+(\S^{p-1})$ by
\begin{align}
	\widehat{W}_2^2(\hat \mu_1,\hat \mu_2) & = \!\!\min_{\mu_1,\mu_2 \in \mathcal P_2(\R^p)}\bigl\{W_2^2(\mu_1,\mu_2): \Pi_2(\mu_1) = \hat \mu_1, \Pi_2(\mu_2) = \hat \mu_2\bigr\}\label{eq:HK-dist}\\
	& = \!\!\min_{\pi \in \mathcal H_\leq (\hat \mu_1,\hat \mu_2)} \int_{(\R^p)^2} \!\! \Vert \vec w_1 - \vec w_2 \Vert^2 \dx \pi(\vec w_1, \vec w_2) + \sum_{i=1}^2 \bigl(\hat \mu_i - \Pi_2(\pi_i)\bigr) (\S^{p-1})\label{eq:HK_relaxed_const}\\
	& = \!\!\min_{\pi \in  \mathcal M^+((\S^{p-1})^2)}\sum_{i=1}^2 \KL(\pi_i, \hat \mu_i) -2 \int_{(\S^{p-1})^2} \log((\boldsymbol \theta_1^T \boldsymbol \theta_2)^+) \dx \pi (\boldsymbol \theta_1, \boldsymbol \theta_2)\label{eq:calib_plan},
\end{align}
where $\mathcal H_\leq (\hat \mu_1,\hat \mu_2) = \{\pi \in \mathcal P_2((\R^p)^2): \Pi_2(\pi_i) \leq \hat \mu_i\}$, and $\pi_i$ denotes the respective marginal of the plan.
It holds that \smash{$\widehat{W}_2$} is a metric on $\mathcal M^+(\S^{p-1})$, which metrizes the weak convergence \cite[Thm.~3.6]{LMS2015}.
Further, $\mathcal M^+(\S^{p-1})$ equipped with this metric is complete, and bounded sets are relatively compact.
Finally, let us remark that
\begin{align}\label{eq:Marg_Mu0}
	\Pi_2\bigl(\mu_0(\!\dx r, \!\dx \boldsymbol \theta)\bigr)= \biggl(\int_{{\R^+}} r^2 \dx \mu_0(r|\boldsymbol \theta) \biggr)\hat \mu_0 (\!\dx \boldsymbol \theta) = \hat \mu_0 (\!\dx \boldsymbol \theta).
\end{align}
Based on these observations, we derive an equivalent formulation for $N(f,\mu_0)$ under the assumption that $\supp(\hat \mu_0) \subset (\S^{p-1})$ covers a sufficiently large part of the space.

\begin{proposition}[Compact-set formulation]\label{prop:CompSetForm}
	Let $\mu_0 \in \mathcal{P}_2 (\R^p)$ such that the corresponding $\hat \mu_0 \in \mathcal{M}^+ (\S^{p-1})$ satisfies
	\begin{equation}\label{eq:CondMu}
		\max_{\boldsymbol \theta_1 \in \S^{p-1}} \min_{\boldsymbol \theta_2 \in \supp(\hat \mu_0)} \mathrm{d}_{\S^{p-1}}(\boldsymbol \theta_1,\boldsymbol \theta_2)< \frac{\pi}{2}.
	\end{equation}
	Then, any $\hat \mu \in \mathcal{M}^+ (\S^{p-1})$ posseses a lift $\mu \in \mathcal P_2(\R^p)$ such that \smash{$\widehat{W}_2(\hat \mu,\hat \mu_0) = W_2(\mu, \mu_0)$}.
	Further, it holds for any $f \in \Hyp$ that
	\begin{equation}\label{eq:complex_simple}
		N(f,\mu_0) = \inf_{\hat \mu \in \mathcal{M}^+ (\S^{p-1})} \biggl\{ \widehat{W}_2^2(\hat \mu,\hat \mu_0): f = \int_{\S^{p-1}} \phi(\boldsymbol \theta,\cdot) \dx \hat \mu(\boldsymbol \theta)\biggr\}.
	\end{equation}
\end{proposition}
\begin{proof}
	First, recall that $\Pi_2(\mu_0) = \hat \mu_0$ due to~\eqref{eq:Marg_Mu0}.
	Based on some $\hat \pi \in \mathcal M^+((\S^{p-1})^2)$ minimizing \smash{$\widehat{W}_2^2(\hat \mu, \hat \mu_0)$} as in~\eqref{eq:calib_plan}, we construct $\mu \in \mathcal{P}_2 (\R^p)$ satisfying $\Pi_2(\mu) = \hat \mu$ and \smash{$\widehat{W}_2(\hat \mu,\hat \mu_0) = W_2(\mu,\mu_0)$}.
	To this end, we make use of the Lebesgue decompositions $\hat \mu = \sigma \hat \pi_2 + \hat \mu^\perp$ and $\hat \mu_0 = \sigma_0 \hat \pi_1 + \hat \mu_0 ^\perp$.
	By \cite[Thm.~6.3b]{LMS2018} and~\eqref{eq:CondMu}, we actually have that $\hat \mu^\perp = 0$.
	
	Now, we define a measurable map $T_{\boldsymbol \theta_1,\boldsymbol \theta_2} \colon \R^2 \to \R^2$ via
	\begin{equation}
		T_{\boldsymbol \theta_1,\boldsymbol \theta_2}(r_1,r_2) = \begin{cases} \Bigl(\sqrt{\frac{\sigma(\boldsymbol \theta_1)}{\sigma_0(\boldsymbol \theta_2)}}r_1,r_2\Bigr) &\mbox{ if } \sigma_0(\boldsymbol \theta_2)>0,\\
			(r_1, r_2) &\mbox{ else.}
		\end{cases}
	\end{equation}
	Using $T_{\boldsymbol \theta_1,\boldsymbol \theta_2}$ and $\mu_0(\!\dx r, \!\dx \boldsymbol \theta)=\mu_0(\!\dx r|\boldsymbol \theta) \hat \mu_0 (\!\dx \boldsymbol \theta) $, we define a lifted measure $\pi \in \mathcal P_2((\R^{p})^2)$ via
	\begin{equation}
		\pi(\!\dx r_1, \!\dx \boldsymbol \theta_1, \!\dx r_2, \!\dx \boldsymbol \theta_2) =   \sigma_0(\boldsymbol \theta_2) T_{\boldsymbol \theta_1,\boldsymbol \theta_2}\strut{}_\#\bigl(\delta_{r_2}(\!\dx r_1) \mu_0(\!\dx r_2| \boldsymbol \theta_2)\bigr) \hat \pi(\!\dx \boldsymbol \theta_1, \!\dx \boldsymbol \theta_2).
	\end{equation}
	First, observe that the marginal $\pi_2$ satisfies for any $\varphi \in C({\R^+} \times \S^{p-1})$ that
	\begin{align}
		&\int_{{\R^+} \times \S^{p-1}} \varphi(r_2,\boldsymbol \theta_2) \dx \pi_2(r_2, \boldsymbol \theta_2) \notag\\
		=&\int_{({\R^+} \times \S^{p-1})^2}  \varphi(r_2,\boldsymbol \theta_2) \sigma_0(\boldsymbol \theta_2)  T_{\boldsymbol \theta_1,\boldsymbol \theta_2}\strut{}_\#\bigl(\delta_{r_2}(\!\dx r_1) \mu_0(\!\dx r_2| \boldsymbol \theta_2)\bigr) \hat \pi(\!\dx \boldsymbol \theta_1, \!\dx \boldsymbol \theta_2)\notag\\
		= & \int_{({\R^+} \times \S^{p-1})^2}  \varphi(r_2,\boldsymbol \theta_2) \sigma_0(\boldsymbol \theta_2)  \delta_{r_2}(\!\dx r_1) \mu_0(\!\dx r_2| \boldsymbol \theta_2) \hat \pi(\!\dx \boldsymbol \theta_1, \!\dx \boldsymbol \theta_2)\notag\\
		= & \int_{\R^{+}\times \S^{p-1}}   \varphi(r_2,\boldsymbol \theta_2) \sigma_0(\boldsymbol \theta_2) \mu_0(\!\dx r_2| \boldsymbol \theta_2) \hat \pi_2(\!\dx \boldsymbol \theta_2),
	\end{align}
	which implies that $\pi_2(\!\dx r, \! \dx \boldsymbol \theta) = \sigma_0(\boldsymbol \theta) \mu_0(\!\dx r| \boldsymbol \theta) \hat \pi_2(\!\dx \boldsymbol \theta)$.
	Due to $\int_{\R} r^2 \mu_0(\!\dx r|\boldsymbol \theta) = 1$ for $\hat \mu_0$-a.e.\ $\boldsymbol \theta \in \S^{p-1}$, we further obtain that $\Pi_2 (\pi_2) = \sigma_0 \hat \pi_2$.
	Again by \cite[Thm.~6.3b]{LMS2018}, there exists a Borel set $A \subset \supp(\pi_2)$ with $\pi_2 (X \setminus A) = 0$ and $\sigma_0(\boldsymbol \theta)>0$ for all $\boldsymbol \theta \in A$.
	Hence, we get for any $\varphi \in C(\S^{p-1})$ that
	\begin{align}
		&\int_{\S^{p-1}} \varphi(\boldsymbol \theta_1) \dx [\Pi_2 (\pi_1)] (\boldsymbol \theta_1)\notag\\
		= &  \int_{({\R^+} \times \S^{p-1})^2} \varphi(\boldsymbol \theta_1) r_1^2 \sigma_0(\boldsymbol \theta_2)  T_{\boldsymbol \theta_1,\boldsymbol \theta_2}\strut{}_\#\bigl(\delta_{r_2}(\!\dx r_1) \mu_0(\!\dx r_2| \boldsymbol \theta_2)\bigr) \hat \pi(\!\dx \boldsymbol \theta_1, \!\dx \boldsymbol \theta_2)\notag\\
		= &  \int_{({\R^+} \times \S^{p-1}) \times ({\R^+} \times \{\boldsymbol \theta_2 : \sigma_0(\boldsymbol \theta_2) >0\})} \varphi(\boldsymbol \theta_1) r_1^2 \sigma(\boldsymbol \theta_1) \delta_{r_2}(\!\dx r_1) \mu_0(\!\dx r_2| \boldsymbol \theta_2) \hat \pi(\!\dx \boldsymbol \theta_1, \!\dx \boldsymbol \theta_2)\notag\\
		= &  \int_{\S^{p-1} \times \{\boldsymbol \theta_2 : \sigma_0(\boldsymbol \theta_2) >0\} } \varphi(\boldsymbol \theta_1) \sigma(\boldsymbol \theta_1) \dx \hat \pi(\boldsymbol \theta_1, \boldsymbol \theta_2)\notag\\
		= &  \int_{\S^{p-1}} \varphi(\boldsymbol \theta_1) \sigma(\boldsymbol \theta_1) \dx \hat \pi_1(\boldsymbol \theta_1),
	\end{align}
	which implies that $\Pi_2(\pi_1) = \sigma \hat \pi_1$.
	Further, it holds that
	\begin{align}
		&\int_{({\R^+} \times \S^{p-1})^2} r_1^2  + r_2^2 -2r_1r_2 \boldsymbol \theta_1^T \boldsymbol \theta_2 \dx \pi (r_1, \boldsymbol \theta_1, r_2, \boldsymbol \theta_2)\notag\\
		= & \int_{(\S^{p-1})^2} \sigma(\boldsymbol \theta_1)  + \sigma_0(\boldsymbol \theta_2) -2 \sqrt{\sigma(\boldsymbol \theta_1)\sigma_0(\boldsymbol \theta_2)} \boldsymbol \theta_1^T \boldsymbol \theta_2 \dx \hat \pi (\boldsymbol \theta_1, \boldsymbol \theta_2)\notag\\
		= & \int_{({\R^+} \times \S^{p-1})^2} r_1^2  + r_2^2 -2r_1r_2 \boldsymbol \theta_1^T \boldsymbol \theta_2 \dx\bigl[\bigl(\sigma(\boldsymbol \theta_1)^{1/2},\boldsymbol \theta_1, \sigma_0(\boldsymbol \theta_2)^{1/2},\boldsymbol \theta_2\bigr)_\#\hat \pi\bigr] (r_1, \boldsymbol \theta_1, r_2, \boldsymbol \theta_2).
	\end{align}
	By \cite[Thm.~7.20iii]{LMS2018}, this implies that $\pi$ is optimal for \smash{$\widehat{W}_2(\hat \mu, \hat \mu_0)$} as in \eqref{eq:HK_relaxed_const}.
	Next, note that the measure
	\begin{equation}
		\pi^\perp(\!\dx r_1, \!\dx \boldsymbol \theta_1, \!\dx r_2, \!\dx \boldsymbol \theta_2) =
		\delta_{0}(\!\dx r_1) \mu_0(\!\dx r_2|\boldsymbol \theta_2) \hat \mu_0^\perp(\!\dx \boldsymbol \theta_2)
	\end{equation}
	satisfies $\Pi_2 (\pi_1^\perp) = 0$, $\pi_2^\perp(\!\dx r, \!\dx \boldsymbol \theta) = \mu_0(\!\dx r|\boldsymbol \theta) \hat \mu_0^\perp(\!\dx \boldsymbol \theta)$, and $\Pi_2 (\pi_2^\perp) = \hat \mu_0^\perp$.
	Since
	\begin{equation}
		\int_{({\R^+} \times \S^{p-1})^2} r_1^2  + r_2^2 -2r_1r_2 \boldsymbol \theta_1^T \boldsymbol \theta_2 \dx \pi^\perp(r_1,\boldsymbol \theta_1, r_2, \boldsymbol \theta_2) = \hat \mu_0^\perp(\S^{p-1}),
	\end{equation}
	we get that $\pi + \pi^\perp$ is an optimal plan for \smash{$\widehat{W}_2(\hat \mu, \hat \mu_0)$} as in \eqref{eq:HK-dist} with the required properties. 
	
	For the second part, we conclude from~\eqref{eq:equiv_constr} and~\eqref{eq:HK-dist} that
	\begin{equation}\label{eq:est_Comp}
		N(f,\mu_0) \geq \inf_{\hat \mu \in \mathcal{M}^+ (\S^{p-1})} \biggl\{ \widehat{W}_2^2(\hat \mu,\hat \mu_0): f = \int_{\S^{p-1}} \phi(\boldsymbol \theta,\cdot) \dx \hat \mu(\boldsymbol \theta)\biggr\}.
	\end{equation}
	Due to the established existence of lifts $\mu \in \mathcal P_2(\R^p)$, \eqref{eq:est_Comp} is sharp and can be replaced by an equality.
\end{proof}

Using Proposition~\ref{prop:CompSetForm}, which requires~\eqref{eq:CondMu} to hold, we can prove the existence of minimizers for~\eqref{eq:N_relax}, namely, that the complexity measure is realized by some $\mu \in \mathcal{P}_2 (\R^p)$.

\begin{lemma}[Minimizing element]\label{lem:exist}
	Let $f \in \Hyp$ and $\mu_0 \in \mathcal{P}_2 (\R^p)$ satisfy~\eqref{eq:CondMu}.
	Then, there exists $\mu \in \mathcal{P}_2 (\R^p)$ with $N(f,\mu_0) = W_2(\mu,\mu_0)^2$ and $f = \int_{\R^p} \phi(\vec w,\cdot) \dx \mu(\vec w)$.
\end{lemma}
\begin{proof}
	By Proposition~\ref{prop:CompSetForm}, it suffices to show existence for~\eqref{eq:complex_simple} since optimal lifts to $\mathcal{P}_2 (\R^p)$ do exist.
	Let $\{\hat \mu_k\}_{k \in \N} \subset \mathcal{M}^+ (\S^{p-1})$ be a minimizing sequence.
	As any such sequence lies in a relatively compact set, we can extract a weakly convergent subsequence with limit $\hat \mu \in \mathcal{M}^+ (\S^{p-1})$.
	Since \smash{$\widehat{W}_2^2(\cdot,\hat \mu_0)^2$} is weakly continuous and $\phi(\cdot,\vec x) \in C(\S^{p-1})$, we get, by definition of the weak convergence, that $\hat \mu$ is a minimizing element.
\end{proof}

To conclude this section, we prove some additional properties of $N$.
\begin{lemma}[Variational properties]\label{lem:UsefulProp}
	The complexity measure $N$ has the following properties.
	\begin{itemize}
		\item[i)] For any $f \in \Hyp$ and $\mu_0, \nu_0 \in \mathcal{P}_2 (\R^p)$ it holds that
		\begin{equation}\label{eq:LipN}
			\bigl \lvert \sqrt{N(f,\mu_0)} - \sqrt{N(f,\nu_0)} \bigr \rvert \leq W_2(\nu_0,\mu_0).
		\end{equation}
		\item[ii)] Let $\mu_0 \in \mathcal{P}_2 (\R^p)$.
		For any  $\{f_k\}_{k \in \N} \subset C(\R^d)$ with $f_k \to f \in C(\R^d)$ pointwise, it holds $N(f,\mu_0) \leq \liminf_{k \to \infty} N(f_k,\mu_0)$.
		\item[iii)] For any $\mu_0 \in \mathcal{P}_2 (\R^p)$, the functional $N(\cdot,\mu_0)$ is convex.
		If $\mu_0$ is absolutely continuous with respect to the Lebesgue measure $\lambda$ and satisfies~\eqref{eq:CondMu}, then  $N(\cdot,\mu_0)$ is strictly convex.
		\item[iv)] For all $f \in \Hyp$ and $\vec x \in \R^d$, it holds that
		\begin{equation}
			N(f,\delta_0) \geq \frac{f(\vec x)}{\Vert \phi(\cdot,\vec x) \Vert_{C(\S^{p-1})}}.
		\end{equation}
	\end{itemize}
	
\end{lemma}
\begin{proof}
	i) Let $\mu \in \mathcal{P}_2 (\R^p)$ with $f = \int_{\R^p} \phi(\vec w,\cdot) \dx \mu(\vec w)$.
	By definition of $N$, we get that
	\begin{equation}
		\sqrt{N(f,\mu_0)} \leq W_2(\mu,\mu_0) \leq W_2(\mu,\nu_0) + W_2(\nu_0,\mu_0).
	\end{equation}
	Taking the infimum over all such $\mu$, we get that $\sqrt{N(f,\mu_0)} \leq \sqrt{N(f,\nu_0)} + W_2(\nu_0,\mu_0)$ which, by symmetry of $W_2$, implies
	\eqref{eq:LipN}.
	
	ii) First, we can assume that $\{N(f_k,\mu)\}_{k \in \N}$ has a bounded subsequence (the statement is trivial otherwise).
	Let $\{\hat \mu_k\}_{k \in \N} \subset \mathcal{M}^+ (\S^{p-1})$ satisfy $f_k = \int_{\S^{p-1}} \phi(\boldsymbol \theta,\cdot) \dx \hat \mu_k(\boldsymbol \theta)$ and \smash{$N(f_k,\mu_0) + 1/k \geq \widehat{W}_2^2(\hat \mu_k,\hat \mu_0)^2$}.
	Hence, there exists a weakly convergent subsequence $\{\hat \mu_{k_j}\}_{j \in \N}$ with \smash{$\liminf_{k \to \infty} N(f_{k},\mu_0) = \lim_{j \to \infty}  \widehat{W}_2^2(\hat \mu_{k_j},\hat \mu_0)^2$}.
	Since \smash{$\widehat{W}_2^2(\cdot,\hat \mu_0)^2$} is weakly continuous and $\phi(\cdot,\vec x) \in C(\S^{p-1})$, we further get that its limit $\hat \mu \in \mathcal{M}^+ (\S^{p-1})$ satisfies that \smash{$\widehat{W}_2^2(\hat \mu,\hat \mu_0)^2 \leq \liminf_{k \to \infty} N(f_k,\mu_0)$} and $f = \int_{\S^{p-1}} \phi(\boldsymbol \theta,\cdot) \dx \hat \mu(\boldsymbol \theta)$.
	Hence, it holds that $N(f,\mu_0) \leq \liminf_{k \to \infty} N(f_k,\mu_0)$.
	
	iii) Let $\lambda \in (0,1)$, $f_1,f_2 \in \Hyp$, and $\epsilon>0$.
	Then, there exist $\mu_1, \mu_2 \in \mathcal{P}_2 (\R^p)$ with $f_i = \int_{\R^p} \phi(\vec w,\cdot) \dx \mu_i(\vec w)$ and $W_2^2(\mu_i,\mu_0) \leq N(f_i,\mu_0) + \epsilon$.
	Further, it is well-known that $W_2^2(\cdot,\mu_0)$ is convex.
	Consequently, we get that
	\begin{align}
		N\bigl(\lambda f_1 + (1-\lambda) f_2,\mu_0\bigr) &\leq W_2^2\bigl(\lambda \mu_1 + (1-\lambda) \mu_2,\mu_0\bigr)\notag\\
		&\leq \lambda N(f_1,\mu_0) + (1-\lambda)N(f_2,\mu_0) + \epsilon.
	\end{align}
	Convexity follows by taking $\epsilon \to 0$.
	If $\mu_0$ is absolutely continuous with respect to $\lambda$ and if~\eqref{eq:CondMu} holds, then we can choose $\epsilon = 0$ and the result follows similarly as before due to the strict convexity of $W_2^2(\cdot,\mu_0)$ in this setting.
	
	iv) Let $\mu \in \mathcal{P}_2 (\R^p)$ satisfy $f(\vec x) = \int_{\R^p} \phi(\vec w,\vec x) \dx \mu(\vec w)$.
	Then, we estimate
	\begin{align}
		f(\vec x) =  \int _{\R^{p}} \phi(\vec w/\Vert \vec w \Vert_2,x) \Vert \vec w \Vert_2^2 \dx \mu (w) \leq  \Vert \phi(\cdot,\vec x) \Vert_{C(\S^{p-1})} \int _{\R^{p}} \Vert \vec w \Vert_2^2 \dx \mu (\vec w).
	\end{align}
	Hence, we conclude that $W_2(\mu,\delta_0)^2 \geq f(\vec x)/ \Vert \phi(\cdot,\vec x) \Vert_{C(\S^{p-1})}$ and the claim follows.
\end{proof}

\section{Interpolating Between the Rich and Kernel Regimes}\label{sec:Interp}
As discussed in Section~\ref{sec:Intro}, it is known that in specific settings the gradient flow~\eqref{eq:GF} converges to the rich regime~\eqref{eq:Rich} for small initializations and to the kernel regime~\eqref{eq:NTK} for large initializations.
In this section, we show that the scaling path~\eqref{eq:ProxPath} interpolates continuously between these two endpoints as $\alpha$ varies from $0$ to $+\infty$.
To this end, we assume that we are given training samples $(\vec x_k,y_k) \in \R^d \times \R$, $k=1,\ldots,n$, such that $\nabla_{\vec w}  \phi(\cdot,\vec x_k) \in L^2(\R^p,\mu_0)$,  $k=1,\ldots,n$, are linearly independent.
For the choice $\phi_{\mathrm{ReLU}}$ from Remark~\ref{rem:ReLU}, this is, for example, the case if the locations $\vec x_k$ of the training samples are distinct.
Then, we can formulate a corresponding regularized learning problem
\begin{equation}\label{eq:ConstrProb}
	\argmin_{f \in C(\R^d)} \sum_{k=1}^n L\bigl(f(\vec x_k),y_k\bigr) + \lambda(1 + \alpha^2)N(f,{S_\alpha}_\#\mu_0),
\end{equation}
where $\alpha \in [0,\infty)$ is an interpolation parameter, $\lambda>0$ is a regularization parameter, and the loss $L(\cdot, y_k)$ is proper, convex, and lower-semicontinuous for every $k=1,\ldots,n$.
\begin{remark}
	All of the results in this section remain true if we investigate
	\begin{equation}
		\argmin_{f \in C(\R^d)} \sum_{k=1}^n L\bigl(f(\vec x_k),y_k\bigr) \quad \text{s.t. } (1 + \alpha^2)N(f,{S_\alpha}_\#\mu_0) \leq \delta
	\end{equation}
	with $L$ strictly convex.
	If $L$ is only convex, then the uniqueness results do not hold.
\end{remark}
For instance, Problem~\eqref{eq:ConstrProb} includes classification problems with $y_k \in \{-1,1\}$ and
\begin{equation}\label{eq:ConstrProbA}
	\argmin_{f \in C(\R^d)} N(f,{S_\alpha}_\#\mu_0) \qquad \text{s.t. } y_k f(\vec x_k) \geq 1 \quad \forall k=1,\ldots,n,
\end{equation}
as well as interpolation problems with $y_k \in \R$ and
\begin{equation}\label{eq:ConstrProB}
	\argmin_{f \in C(\R^d)} N(f,{S_\alpha}_\#\mu_0) \qquad \text{s.t. } f(\vec x_k) = y_k \quad \forall k=1,\ldots,n
\end{equation}
as special cases.
When $L$ is the square loss, the interpolation problem~\eqref{eq:ConstrProB} can be interpreted as the endpoint of the modified regularization path~\eqref{eq:ProxPath} as $\lambda \to 0$.
Instead of~\eqref{eq:ConstrProb}, we can also investigate the \emph{equivalent} parameter-space problems
\begin{equation}\label{eq:ConstrProbMeas1}
	\argmin_{\mu \in \mathcal P_2 (\R^{p})} \sum_{k=1}^n L\biggl(\int_{\R^{p}} \phi(\vec w,\vec x_k) \dx \mu(\vec w),y_k\biggr) + \lambda (1 + \alpha^2) W_2^2(\mu, {S_\alpha}_\#\mu_0)
\end{equation}
and
\begin{equation}\label{eq:ConstrProbMeas2}
	\argmin_{\hat \mu \in \mathcal{M}^+ (\S^{p-1})} \sum_{k=1}^n L\biggl(\int_{\S^{p-1}} \phi(\boldsymbol \theta,\vec x_k) \dx \hat \mu(\boldsymbol \theta),y_k\biggr) + \lambda (1 + \alpha^2) \widehat{W}_2^2\bigl(\hat \mu, \alpha^2 \hat \mu_0\bigr).
\end{equation}
These reformulations are essential to prove the continuity of the optimal solutions $f^*_\alpha$ for~\eqref{eq:ConstrProb} with respect to $\alpha$ in Theorem~\ref{cor:ConvSol}.
First, however, we establish the existence of minimizers for~\eqref{eq:ConstrProb}.
\begin{lemma}
	Assume that~\eqref{eq:ConstrProb} is feasible. Then, there exists a minimizer.
	If, additionally, $\mu_0$ is absolutely continuous with respect to the Lebesgue measure $\lambda$ and satisfies~\eqref{eq:CondMu}, then the minimizer is unique for $\alpha>0$.
\end{lemma}
\begin{proof}
	Let $\{f_k\}_{k \in \N}$ be a minimizing sequence, which implies that the corresponding sequence $\{N(f_k,{S_\alpha}_\#\mu_0)\}_{k \in \N}$ is bounded.
	Similarly as in the proof of Lemma~\ref{lem:UsefulProp}ii), we can extract a subsequence $\{f_{k_j}\}_{j \in \N}$ such that there is a $\mu \in \mathcal P_2(\R^p)$ with $N(f_{k_j},{S_\alpha}_\#\mu_0) \to W_2^2(\mu, \mu_0)$ and $f_{k_j} \to f = \int_{\R^p} \phi(\vec w,\cdot) \dx \mu(\vec w)$ point-wise.
	However, this readily implies that $\sum_{l=1}^n L(f(\vec x_l),y_l) \leq  \liminf_{j \to \infty} \sum_{l=1}^n L(f_{k_j}(\vec x_l),y_l)$ and further that $N(f,{S_\alpha}_\#\mu_0) \leq \lim_{j \to \infty} N(f_{k_j},{S_\alpha}_\#\mu_0)$.
	Hence, we get that $f$ is a minimizer.
	If the additional assumptions hold, uniqueness follows by strict convexity (see Lemma~\ref{lem:UsefulProp}).
\end{proof}
\begin{remark}
	A similar statement also holds for the formulations~\eqref{eq:ConstrProbMeas1} and~\eqref{eq:ConstrProbMeas2}.
\end{remark}
Now, we investigate the behavior of the functional in~\eqref{eq:ConstrProbMeas2} as $\alpha$ varies.
To this end, we rely on the concept of $\Gamma$-convergence (see \cite{Braides02} for a detailed exposition).
Let $\X$ be a topological space.
Recall that $\{J_k\}_{k\in\N}$ 
with $J_k\colon \X \rightarrow [0,\infty]$ is said to $\Gamma$-converge to $J \colon  \X \rightarrow [0,\infty]$ 
if the following two conditions are fulfilled for every $x \in \X$:
\begin{enumerate}
	\item[i)] it holds that $J(x) \leq \liminf_{k \rightarrow \infty} J_k(x_k)$ whenever  $x_k  \to  x$;
	\item[ii)] there is a sequence $\{x_k\}_{k\in\N}$ with $x_k \to x$ and $\limsup_{k \to \infty} J_k(x_k) \le J(x)$.
\end{enumerate}
The importance of $\Gamma$-convergence is captured by Theorem \ref{thm:FundGamma}.
Recall that a family of functionals $J_k \colon \X \to \R$ is equicoercive if it is bounded from below by a coercive functional.
\begin{theorem}[Theorem of $\Gamma$-convergence \cite{Braides02}]\label{thm:FundGamma}
	Let $\{J_k\}_{k \in \N}$ be an equi\-coercive family of functionals $J_k \colon \X \to \R$.
	If $J_k$ $\Gamma$-converges to $J$, then it holds that
	\begin{itemize}
		\item the optimal functional values converge $\lim_{k \to \infty} \inf_{x \in \X } J_k(x) = \inf_{x \in \X} J(x)$; 
		\item all accumulation points of the minimizers of $J_k$ are minimizers of $J$.
	\end{itemize}
\end{theorem}
Although, Theorem~\ref{thm:FundGamma} and the next two paragraphs on $\Gamma$-convergence of the functional in \eqref{eq:ConstrProbMeas2} might appear quite abstract at first glance, they will ultimately enable us to prove continuity of the optimal solutions $f^*_\alpha$ for \eqref{eq:ConstrProb} with respect to $\alpha$ in our main Theorem~\ref{cor:ConvSol}.

\paragraph{Rich Regime}
Using $\Gamma$-convergence, we first investigate the case $\alpha \to \alpha_* \neq \infty$ and equip $\mathcal{M}^+ (\S^{p-1})$ with the usual weak topology.
\begin{proposition}\label{lem:Gamma0}
	For $\alpha \to \alpha_* \neq \infty$, we have $\Gamma$-convergence of the functionals in
	\begin{equation}\label{eq:ConstrProbMeas2.1}
		\min_{\hat \mu \in \mathcal{M}^+ (\S^{p-1})} \sum_{k=1}^n L\biggl(\int_{\S^{p-1}} \phi(\boldsymbol \theta,\vec x_k) \dx \hat \mu(\boldsymbol \theta),y_k\biggr) + \lambda (1 + \alpha^2) \widehat{W}_2^2\bigl(\hat \mu, \alpha^2 \hat \mu_0\bigr).
	\end{equation}
	Furthermore, the family of functionals in~\eqref{eq:ConstrProbMeas2.1} is equicoercive.
\end{proposition}
\begin{proof}
	We first note that the functionals in~\eqref{eq:ConstrProbMeas2.1} are equicoercive since
	\begin{equation}
		(1 + \alpha^2) \widehat{W}_2^2\bigl(\hat \mu, \alpha^2 \hat \mu_0\bigr) \geq \Bigl( \widehat{W}_2(\hat \mu, 0) -\widehat{W}_2\bigl(0, \alpha^2 \hat \mu_0\bigr)\Bigr)^2 = \bigl( \widehat{W}_2(\hat \mu, 0) - \alpha \sqrt{\hat \mu_0(\S^{p-1})}\bigr)^2
	\end{equation}
	and $L$ maps into $[0,\infty]$.
	For the $\liminf$ inequality, let $\{\hat \mu_k\}_{k \in \N}$ and $\{\alpha_k\}_{k \in \N}$ be sequences with limits $\hat \mu$ and $\alpha_*$, respectively.
	Since $\phi(\cdot,\vec x_l)$ is continuous, this directly implies that $\int_{\S^{p-1}} \phi(\boldsymbol \theta,\vec x_l) \dx \hat \mu_k(\boldsymbol \theta) \to \int_{\S^{p-1}} \phi(\boldsymbol \theta,\vec x_l) \dx \hat \mu(\boldsymbol \theta)$ for all $l=1,\ldots,n$.
	Then, since
	\begin{align}
		(1 + \alpha_k^2) \widehat{W}_2^2\bigl(\hat \mu_k, \alpha_k^2 \hat \mu_0\bigr) &\geq (1 + \alpha_k^2) \Bigl( \widehat{W}_2\bigl(\hat \mu_k, \alpha_*^2 \hat \mu_0\bigr) -\widehat{W}_2\bigl(\alpha_*^2 \hat \mu_0, \alpha_k^2 \hat \mu_0\bigr)\Bigr)^2\\
		& \geq (1 + \alpha_k^2) \Bigl( \widehat{W}_2\bigl(\hat \mu_k, \alpha_*^2 \hat \mu_0\bigr) - \sqrt{\vert \alpha_*^2 - \alpha_k^2 \vert \hat \mu_0(\S^{p-1})}\Bigr)^2,
	\end{align}
	the claim follows  by the continuity of \smash{$\widehat{W}_2$} and the lower-semicontinuity of $L(\cdot,y_l)$.
	Finally, the $\limsup$ inequality follows if we let the recovery sequence be constant.
\end{proof}
Note that for $\alpha = 0$, problem \eqref{eq:ConstrProbMeas2.1} can be rewritten as
\begin{equation}
	\min_{\hat \mu \in \mathcal{M}^+ (\S^{p-1})} \sum_{k=1}^n L\biggl(\int_{\S^{p-1}} \phi(\boldsymbol \theta,\vec x_k) \dx \hat \mu(\boldsymbol \theta),y_k\biggr) + \lambda \TV(\hat \mu) \label{eq:RichRegime}.
\end{equation}

\paragraph{Kernel Regime}
Next, we want to discuss the case $\alpha \to \infty$ and show that we approach the NTK problem \eqref{eq:neural_tangent_kernel}
with feature maps $\vec x \mapsto \nabla_{\vec w}  \phi(\cdot,\vec x)$ if we reformulate \eqref{eq:ConstrProbMeas2} accordingly.

\begin{proposition}\label{lem:GammaInf}
	Let $\mu_0$ be absolutely continuous with respect to the Lebesgue measure and $\int_{\R^p} \phi(\vec w,\cdot) \dx \mu_0(\vec w) = 0$.
	Further, let $ L(\cdot,y_k)$, $k=1,\ldots,n$, be either left- or right-continuous in every point of its domain.
	Then, for $\alpha \to \infty$, we have $\Gamma$-convergence of the functionals in
	\begin{equation}
		\argmin_{T \in L^2(\R^p, \mu_0)} \sum_{k=1}^n L\biggl( \int_{\R^p} \phi\bigl(\alpha \vec w+\alpha^{-1}T(\vec w),\vec x_k\bigr) \dx \mu_0(\vec w),y_k\biggr) + \lambda \frac{\alpha^2 + 1}{\alpha^2} \Vert T \Vert_{L^2(\R^p,\mu_0)}^2,\label{eq:ProbT}
	\end{equation}
	which is a reformulation of \eqref{eq:ConstrProbMeas2} using transport maps,
	to the one in
	\begin{equation}\label{eq:KernelSetting}
		\argmin_{T \in L^2(\R^p, \mu_0)} \sum_{i=k}^n L\bigl(\bigl\langle T, \nabla_{\vec w} \phi(\cdot,\vec x_k) \bigr\rangle_{L^2(\R^p, \mu_0)},y_k\bigr) + \lambda \Vert T \Vert_{L^2(\R^p,\mu_0)}^2
	\end{equation}
	with respect to the weak topology in $L^2(\R^p,\mu_0)$.
	Further, the functionals in~\eqref{eq:ProbT} are equicoercive.
\end{proposition}
\begin{proof}
	Equicoercivity of the functionals in~\eqref{eq:ProbT} holds since \smash{$\Vert T \Vert_{L^2(\R^p,\mu_0)}^2$} is a lower bound for all of them.
	Due to the absolute continuity with respect to the Lesbegue measure, we can use the equivalent formulation~\eqref{eq:N} of the complexity measure in~\eqref{eq:N_relax} to obtain
	\begin{alignat}{1}
		&\min_{\mu \in \mathcal P_2 (\R^{p})} \sum_{k=1}^n L\biggl(\int_{\R^{p}} \phi(\vec w,\vec x_k) \dx \mu(\vec w),y_k\biggr) + \lambda(1 + \alpha^2) W_2^2(\mu, {S_\alpha}_\#\mu_0)\notag\\
		=&\min_{T \in L^2(\R^p, \mu_0)} \sum_{k=1}^n L\biggl(\int_{\R^p}\!\! \phi\bigl(\vec w,\vec x_k\bigr) \dx \bigl[{(S_\alpha + \alpha^{-1}T)}_\#\mu_0\bigr](\vec w),y_k\biggr) + \lambda \frac{\alpha^2 + 1}{\alpha^2} \Vert T \Vert_{L^2(\R^p,\mu_0)}^2\notag\\
		=&\min_{T \in L^2(\R^p, \mu_0)} \sum_{k=1}^n L\biggl( \int_{\R^p} \phi\bigl(\alpha \vec w+\alpha^{-1}T(\vec w),\vec x_k\bigr) \dx \mu_0(\vec w),y_k\biggr) + \lambda \frac{\alpha^2 + 1}{\alpha^2} \Vert T \Vert_{L^2(\R^p,\mu_0)}^2.
	\end{alignat}
	Now, since $\phi$ is twice continuously differentiable on $C_{\vec x_k}$, we get, for any $\vec w \in C_{\vec x_k}$, that
	\begin{equation}
		\phi(\vec w+ \vec h,\vec x_k) = \phi(\vec w,\vec x_k) + \bigl(\nabla_{\vec w} \phi(\vec w,\vec x_k)\bigr)^T \vec h + R(\vec w,\vec h,\vec x_k).
	\end{equation}
	As $t \mapsto \phi(\vec w+t\vec h,\vec x_k)$ is absolutely continuous, the remainder $R(\vec w,\vec h,\vec x_k)$ can be estimated for any $\vec w \in C_{\vec x_k}$ and $\vec h \in \R^p$ by 
	\begin{equation}
		|R(\vec w,\vec h,\vec x_k)| \leq \max_{\tilde {\vec w} \in B(\vec w, \Vert \vec h \Vert)} \bigl \Vert \nabla_{\vec w} \phi(\tilde {\vec w},\vec x_k) -\nabla_{\vec w} \phi(\vec w,\vec x_k) \bigr \Vert \Vert \vec h \Vert.
	\end{equation}
	If additionally $\vec h \in B(\vec w, \epsilon)$, where the radius $\epsilon$ depends on $\vec w$ and $\vec x_i$, we can use differentiability to even get
	\begin{equation}
		|R(\vec w,\vec h,\vec x_i)| \leq \sup_{\tilde {\vec w} \in B(\vec w,\epsilon)} \Vert \nabla^2_{\vec w} \phi(\tilde{\vec w},\vec x_i) \Vert \Vert \vec h \Vert^2.
	\end{equation}
	By defining the function
	\begin{equation}
		R_{\vec x_i,\alpha}(T) \coloneqq  \int_{\R^p} R(\alpha \vec w,T(\vec w)/\alpha,\vec x_i) \dx \mu_0(\vec w),
	\end{equation}
	we rewrite \eqref{eq:ProbT} for the following $\Gamma$-convergence discussion as
	\begin{equation}
		\min_{T \in L^2(\R^p, \mu_0)} \sum_{k=1}^n L\Bigl(\bigl\langle T, \nabla_{\vec w}  \phi(\cdot,\vec x_k) \bigr\rangle_{L^2(\R^p, \mu_0)} + R_{\vec x_i,\alpha}(T),y_k\Bigr) + \lambda \frac{\alpha^2 + 1}{\alpha^2} \Vert T \Vert_{L^2(\R^p,\mu_0)}^2.\label{eq:ReformLimitInfty}
	\end{equation}
	
	For the $\liminf$ inequality of $\Gamma$-convergence, let $\{T_k\}_{k \in \N}$ and $\{\alpha_k\}_{k \in \N}$ be (weakly) convergent sequences with limits $T$ and $\infty$, respectively.
	Since weakly convergent sequences are bounded, we get that $\Vert T_k \Vert_{L^2(\R^p,\mu_0)} /\alpha_k \to 0$.
	Hence, we can drop to a subsequence that satisfies $T_k(\vec w) /\alpha_k \to 0$ for $\mu_0$-a.e.\ $\vec w \in \R^p$, and there exists $g \in L^2(\R^p,\mu_0)$ with $\vert T_k(\vec w) /\alpha_k \vert \leq g(\vec w)$ for $\mu_0$-a.e.\ $\vec w \in \R^p$.
	% See https://math.stackexchange.com/questions/714744/l1-convergence-gives-a-pointwise-convergent-subsequence
	Now, observe that
	\begin{align}
		\vert R_{\vec x_l,\alpha_k}(T_k) \vert & \leq  \int_{\R^p} \max_{\tilde{\vec w} \in B(\alpha_k \vec w, g(\vec w))} \! \Vert \nabla_{\vec  w} \phi(\tilde{\vec w},\vec x_l) -\nabla_{\vec w} \phi( \alpha_k \vec w,\vec x_l) \Vert \frac{\Vert T_k \Vert}{\alpha_k} \dx \mu_0(\vec w)\notag\\
		& \leq \biggl(\int_{\R^p} \max_{\tilde{\vec w} \in B(\vec w, \frac{g(\vec w)}{\alpha_k})} \!\Vert \nabla_{\vec w} \phi(\tilde{\vec w},\vec x_l) -\nabla_{\vec w} \phi(\vec w,\vec x_l) \Vert^2 \dx \mu_0(\vec w)\biggr)^{1/2} \! \Vert T_k \Vert_{L^2(\R^p,\mu_0)} \label{eq:est_rem}.
	\end{align}
	Here, the integrand converges pointwise to 0 for every $\vec w \in C_{\vec x_l}$, and can be bounded by
	\begin{align}
		&\max_{\tilde{\vec w} \in B(\vec w, g(\vec w)/\alpha_k)} \Vert \nabla_{\vec w} \phi(\tilde{\vec w},\vec x_l) -\nabla_{\vec w} \phi(\vec w,\vec x_l) \Vert\notag\\
		&\leq \max_{\tilde{\vec w} \in B(\vec w, g(\vec w)/\alpha_k)}  \Vert \tilde{\vec w} \Vert \Bigl\Vert \nabla_{\vec w} \phi\Bigl(\frac{\tilde{\vec w}}{\Vert \tilde{\vec w} \Vert},\vec x_l\Bigr)\Bigr\Vert + \Vert \vec w \Vert \Bigl\Vert \nabla_{\vec w} \phi\Bigl(\frac{\vec w}{\Vert \vec w \Vert},\vec x_l\Bigr) \Bigr\Vert\notag\\
		& \leq \Bigl(2\Vert \vec w \Vert + \frac{g(\vec w)}{\alpha_k}\Bigr) \max_{\tilde{\vec w} \in \S^p} \Vert \nabla_{\vec w} \phi(\tilde{\vec w},\vec x_l) \Vert.
	\end{align}
	From the dominated-convergence theorem, one has that $R_{\vec x_l,\alpha_k}(T_k) \to 0$.
	Given that $T_k \rightharpoonup T$, the $\liminf$ inequality now follows as
	\begin{align}
		&\sum_{l=1}^n \bigl(\bigl\langle T, \nabla_{\vec w}  \phi(\cdot,\vec x_l) \bigr\rangle_{L^2(\mu_0)},y_l\bigr) + \lambda \Vert T \Vert_{L^2(\R^p,\mu_0)}^2\notag\\
		\leq &\liminf_{k \to \infty} \sum_{l=1}^n L\bigl(\bigl\langle T_k, \nabla_{\vec w}  \phi(\cdot,\vec x_l) \bigr\rangle_{L^2(\R^p, \mu_0)} + R_{\vec x_l,\alpha_k}(T_k),y_l\bigr) +  \lambda \frac{\alpha_k^2 + 1}{\alpha_k^2} \Vert T_k \Vert_{L^2(\R^p,\mu_0)}^2.
	\end{align}
	
	For the $\limsup$ inequality, we can assume that $T$ has finite energy.
	Further, we use a dual basis $D_l \in L^2(\R^p,\mu_0)$, $l=1,\ldots,n$, of the feature maps $\nabla_{\vec w}  \phi(\cdot,\vec x_l)$. 
	Then, we define $h(\vec w) \coloneqq \vert T(\vec w) \vert + \sum_{l=1}^n \vert D_l(\vec w) \vert$ and
	\begin{align}
		M_{l,k} = \biggl(\int_{\R^p} \max_{\tilde{\vec w} \in B(\vec w, h(\vec w)/\alpha_k^2)} \Vert \nabla_{\vec w} \phi(\tilde{\vec w},\vec x_l) -\nabla_{\vec w} \phi(\vec w,\vec x_l) \Vert^2 \dx \mu_0(\vec w)\biggr)^{1/2} \Vert \vec h \Vert_{L^2(\R^p,\mu_0)}.
	\end{align}
	Now, set $s_l = 1$ if $L(\cdot,y_l)$ is right-continuous in $\langle T, \nabla_{\vec w}  \phi(\cdot,\vec x_l) \rangle_{L^2(\R^p, \mu_0)}$ and $s_l = -1$ if it is left-continuous.
	Finally, we pick $T_k = T + \sum_l s_l M_{l,k} D_l$ as recovery sequence.
	
	As in the first part of the proof, we can show that $M_{l,k} \to 0$ for $k \to \infty$.
	Hence, we can estimate as in~\eqref{eq:est_rem} and obtain that $R_{\vec x_l,\alpha_k}(T_k) \to 0$ for $k \to \infty$.
	In the right-continuous case, it holds that
	\begin{align}
		\bigl \langle T_k, \nabla_{\vec w}  \phi(\cdot,\vec x_l) \bigr \rangle_{L^2(\R^p, \mu_0)} + R_{\vec x_l,\alpha_k}(T_k) \geq \bigl\langle T, \nabla_{\vec w} \phi(\cdot,\vec x_l) \bigr\rangle_{L^2(\R^p, \mu_0)}.
	\end{align}
	For the left-continuous case, we get that
	\begin{equation}
		\bigl \langle T_k, \nabla_{\vec w}  \phi(\cdot,\vec x_l) \bigr\rangle_{L^2(\R^p, \mu_0)} + R_{\vec x_l,\alpha_k}(T_k) \leq \bigl\langle T, \nabla_{\vec w}  \phi(\cdot,\vec x_l) \bigr\rangle_{L^2(\R^p, \mu_0)}.
	\end{equation}
	Hence, we have for $l=1,\ldots,n$ that $\langle T_k, \nabla_{\vec w}  \phi(\cdot,\vec x_l) \rangle_{L^2(\R^p, \mu_0)} \to \langle T, \nabla_{\vec w}  \phi(\cdot,\vec x_l) \rangle_{L^2(\R^p, \mu_0)}$ from the required direction, which concludes the proof.
\end{proof}
\paragraph{Implications for $\Hyp$}
Assume that $\mu_0$ is absolutely continuous with respect to the Lebesgue measure.
Observe that Proposition~\ref{lem:Gamma0} and Theorem~\ref{thm:FundGamma} directly imply that the family of measures $\hat \mu^*_\alpha$, $\alpha \in (0,\infty)$, determined by ~\eqref{eq:ConstrProbMeas2} is continuous in the \smash{$\widehat W_2$} metric.
Further, for $\alpha \to 0$, these measures converge to some optimal solution $\hat \mu^*_0$ of~\eqref{eq:RichRegime}.
Finally, Proposition~\ref{lem:GammaInf} and Theorem~\ref{thm:FundGamma} imply that the the solutions $T^*_\alpha$ of~\eqref{eq:ProbT} converge to an optimal solution of~\eqref{eq:KernelSetting} in the weak $L^2$-topology.
Equivalently, we can state these observations in terms of the optimal NNs as
\begin{equation}\label{eq:OptFunc}
	f^*_\alpha = \int_{\S^{p-1}} \phi(\boldsymbol \theta,\cdot) \dx \hat \mu^*_\alpha(\boldsymbol \theta) = \int_{\R^p} \phi\bigl(\vec w \theta+\alpha^{-1}T^*_\alpha(\vec w),\cdot \bigr) \dx \mu_0(\vec w), \quad \alpha < \infty
\end{equation}
and
\begin{equation}
	f^*_\infty(\vec x) = \bigl\langle T_\infty^*, \nabla_{\vec w}  \phi(\cdot,\vec x)\bigr\rangle_{L^2(\R^p, \mu_0)},
\end{equation}
where $\mu^*_\alpha$, $T^*_\alpha$, and $T_\infty^*$ solve \eqref{eq:ConstrProbMeas2}, \eqref{eq:ProbT}, and \eqref{eq:KernelSetting}, respectively.
\begin{theorem}\label{cor:ConvSol}
	Assume that $\mu_0$ is absolutely continuous with respect to the Lebesgue measure.
	Then, the family $f^*_\alpha\colon \R^d \to \R$, $\alpha \in [0,\infty)$, of optimal solutions for~\eqref{eq:ConstrProb} is continuous with respect to the uniform norm on any compact set $K \subset \R^d$.
	Further, for $\alpha \to \infty$, we have that $f^*_\alpha \to f^*_\infty$ pointwise.
\end{theorem}
\begin{proof}
	Let $\alpha_k \to \alpha \in  [0,\infty)$ with $\alpha_k \neq 0$.
	From~\eqref{eq:OptFunc} and the weak convergence of the $\hat \mu^*_{\alpha_k}$, we get that the sequence $f^*_{\alpha_k}$ is pointwise-convergent. 
	Further, recall that all $f^*_{\alpha_k}$ are Lipschitz-continuous with constant $C \hat \mu^*_{\alpha_k}(\S^{p-1})$.
	Since weakly convergent sequences have bounded measures, the $f^*_{\alpha_k}$ are uniformly Lipschitz-continuous.
	Hence, we conclude that $\Vert f^*_{\alpha_k} - f^*_\alpha \Vert_{C(K)} \to 0$ for any compact set $K \subset \R^d$.
	For the case $\alpha_k \to \infty$, we have already shown in the proof of Proposition~\ref{lem:GammaInf} that the sequence $f^*_{\alpha_k}$ converges pointwise to $f^*_{\infty}$.
\end{proof}

\section{Path Comparison at Final States}\label{Sec:NumExample}
To illustrate our theoretical observations, we investigate a 2D interpolation problem with samples $(\tilde{\vec x}_k,y_k) \in [-1,1]^2 \times \{-1,1\}$, $k=1,\ldots,10$, which are depicted in Figure~\ref{fig:data}.
\begin{figure}[t]
	\centering
	\includegraphics[width=0.45\textwidth]{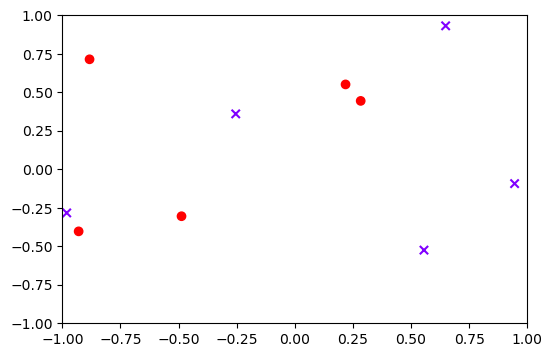}
	\caption{Data for Problem \eqref{eq:Interp_prob} (crosses corresponds to $y_k=-1$, circles to $y_k=1$).}\label{fig:data}
\end{figure}
Note that we have chosen to investigate an interpolation problem as it describes the end states of both the static and the dynamic paths.
As discussed in Remark~\ref{rem:ReLU}, we modify the $x$-component of these samples to $\vec x_k = (\tilde{\vec x}_k,1) \in \R^3$ in order to use a 2-homogeneous infinite-width 2-layer ReLU NN model with parameterization function $\phi_{\mathrm{ReLU}}\colon \R^4 \times \R^3 \to \R$ given by
\begin{equation}
	\phi_{\mathrm{ReLU}}(\vec w, \vec x) = w_1 \text{ReLU}(\vec w_2^T \vec x),
\end{equation}
where $\vec w = (w_1, \vec w_2) \in \R \times \R^{3}$.
Now, our goal is to find a probability measure $\mu \in \mathcal P_2(\R^4)$ such that 
\begin{equation}\label{eq:Interp_prob}
	\int_{\R^{4}}\phi_{\mathrm{ReLU}}\bigl(\vec w,\vec x_k\bigr) \dx \mu(\vec w) = y_k, \qquad k=1,\ldots,n.
\end{equation}
By using atomic measures, any finite-width 2-layer ReLU NN $\Psi(\vec x) = \sum_{k=1}^n w_k (\vec v_k^T \vec  x)^+$ with $w_k \in \R$ and $\vec v_k \in \R^3$ is covered by this formulation.
Further, \eqref{eq:Interp_prob} can be recast as the search for a measure $\hat \mu \in \mathcal M^+(\S^{3})$ (see \eqref{eq:equiv_constr}).
But, even then,~\eqref{eq:Interp_prob} is in general under-determined and we need to employ some kind of explicit or implicit regularization in order to ensure \emph{nice} solutions.

\subsection{Scaling Path}\label{sec:VarProb}
First, we look into the solution of the variational problem~\eqref{eq:ConstrProbMeas2} which, for the described interpolation setting, reads
\begin{equation}\label{eq:InterpolationProb}
	\argmin_{\hat \mu \in \mathcal{M}^+ (\S^{3})} \widehat{W}_2^2\bigl(\hat \mu, \alpha^2 \hat \mu_0\bigr) \quad \text{s.t. } \int_{\S^{3}} \phi_{\mathrm{ReLU}}(\boldsymbol \theta,\vec x_k) \dx \hat \mu(\boldsymbol \theta) = y_k, \quad k=1,\ldots,10,
\end{equation}
with $\boldsymbol \theta = (\theta_1, \boldsymbol \theta_2) \in \S^3 \subset \R^4$.
The initialization $\hat \mu_0$ is chosen as the uniform measure on $P = \{\pm 1/\sqrt 2\} \times \S^2/\sqrt 2 \subset \S^3$.
The choice of a uniform measure on $P$ instead of $\S^3$ is motivated, on the one hand, by the training dynamics and, on the other hand, by the initialization of the dynamic viewpoint based on the gradient flow \eqref{eq:GF} investigated in Section~\ref{sec:GradDynamics}.
To compute \smash{$\widehat{W}_2^2$}, we make use of the formulation~\eqref{eq:calib_plan}.
More precisely, this corresponds to the unbalanced optimal transport
\begin{equation}
	\widehat{W}_2(\hat \mu_1, \hat \mu_2)^2 = \min_{\gamma  \in \mathcal M^+(\S^{3}\times \S^{3})} \int_{\S^{3}\times \S^{3}}
	c \dx \gamma  +
	\sum^2_{i=1}
	\text{KL}({\pi_i}_\sharp \gamma, \hat \mu_i),\label{eq:UnbalancedOT}
\end{equation}
where
\begin{equation}
	c(\boldsymbol \theta_1,\boldsymbol \theta_2) =
	\begin{cases}
		-2 \log(\boldsymbol \theta_1^T \boldsymbol \theta_2) & \mbox{if } \boldsymbol \theta_1^T \boldsymbol \theta_2 >0,\\
		\infty & \mbox{else.}
	\end{cases}
\end{equation}
Problem~\eqref{eq:InterpolationProb} is an infinite-dimensional convex-optimization problem. To make it computationally tractable, we discretize the spheres $\S^2$ in $P$ using a Fibonacci grid $\vec F$ with $250^2$ points \cite{SwiJam06}.
The corresponding discrete version of $P$ and the measure $\hat \mu_0$ are denoted by $\vec P$ and $\hat{\boldsymbol \mu}_0$, respectively.
Additionally, we discretize the search space for $\mu$ based on the Fibonacci grid $\vec F$ as
\begin{equation}
	\vec Q = \biggl\{\frac{1}{7 \sqrt{2}}\Bigl(\pm k, \bigl(98 - k^2\bigr)^{1/2} \vec x\Bigr): 1 \leq k \leq 9, \vec x \in \vec F \biggr\} \subset \S^3.
\end{equation}
The coarser discretization in the first coordinate of the set $\vec Q$ is motivated by the fact that $\phi_{\mathrm{ReLU}}((w_1,\vec w_2),\vec x) = \phi_{\mathrm{ReLU}}((w_1/\vert w_1 \vert,\vec w_2\vert w_1 \vert),\vec x)$, namely, that the model is considerably over-parameterized.
This choice also ensures that $\vec P \subset \vec Q$.
Now, we obtain a discrete convex problem involving the unbalanced optimal-transport distance \smash{$\widehat{W}_2$}, which is still computationally challenging due to its large size.
Therefore, we resort to an entropy-regularized distance \smash{$\widehat{W}_{2,\epsilon}$} (see \cite{FSVATP2018}) instead of the original formulation \eqref{eq:UnbalancedOT}.
The divergence \smash{$\widehat{W}_{2,\epsilon}$} can be computed efficiently through the Sinkhorn algorithm, and its' gradients can be computed using algorithmic differentiation.
For small regularization parameters such as $\epsilon=\num{1e-2}$, the approximation \smash{$\widehat{W}_{2,\epsilon}$} is reasonably close to the original \smash{$\widehat{W}_2$} distance \cite{FSVATP2018,NS20}.
Finally, we arrive at the fully discrete problem
\begin{equation}\label{eq:TraingDiscrete}
	\argmin_{\hat{\boldsymbol \mu} \in \mathcal{M}^+ (\vec Q)} \widehat{W}_{2,\epsilon}^2\bigl(\hat{\boldsymbol \mu}, \alpha^2 \hat{\boldsymbol \mu}_0\bigr) \quad \text{s.t. } \int_{ \vec Q} \phi_{\mathrm{ReLU}}(\boldsymbol \theta,\vec x_k) \dx \hat{\boldsymbol \mu}(\boldsymbol \theta) = y_k, \quad k=1,\ldots,10,
\end{equation}
which amounts to the minimization of a differentiable convex objective subject to linear equality constraints.
Such problems can be solved, for example, with the forward-backward splitting \cite{CW05}.
To ensure fast convergence, we couple this method with a spectral step-size predictor and an Armijo linesearch to ensure convergence as detailed in \cite{GolStu2014}.
To evaluate \smash{$\widehat{W}_{2,\epsilon}^2\bigl(\cdot, \alpha^2 \hat{\boldsymbol \mu}_0\bigr)$} and its gradients, we make use of the \emph{geomloss} package\footnote{\href{https://www.kernel-operations.io/geomloss/}{https://www.kernel-operations.io/geomloss/}}.
Our numerical results for various values of $\alpha$ (including the limiting cases $\alpha=0$ and $\alpha= \infty$) are depicted in Figure~\ref{fig:interpolation}.
We clearly observe that a larger regularization scale $\alpha$ leads to smoother solutions.
Additionally, we observe that the $f_\alpha^*$ converge visually for $\alpha \to 0$ and $\alpha \to \infty$, as predicted by Corollary~\ref{cor:ConvSol}.
The corresponding functional values multiplied by the correct scaling $1+\alpha^2$ can be found in Table~\ref{tab:interpolation}.
For the NTK setting, the optimal value corresponding to~\eqref{eq:KernelSetting} is $\num{2.83e2}$. 
Again, we observe convergence of \smash{$(1+\alpha^2)\widehat{W}_{2,\epsilon}^2\bigl( \hat{\boldsymbol \mu}^*_\alpha, \alpha^2 \hat{\boldsymbol \mu}_0\bigr)$}, as predicted by Propositions~\ref{lem:Gamma0} and~\ref{lem:GammaInf}, and Theorem~\ref{thm:FundGamma}.
\setlength{\tabcolsep}{1pt}
\begin{figure}[t]
	\centering
	\begin{tabular}{cccc}
		$\alpha=0$  &  $\alpha=10^{-1}$   & $\alpha=10^{0}$ \\
		\includegraphics[width=0.31\textwidth]{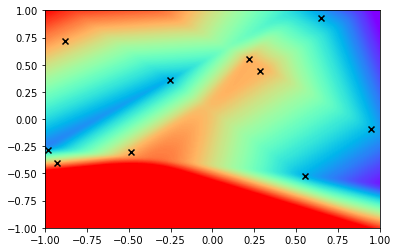} & 
		\includegraphics[width=0.31\textwidth]{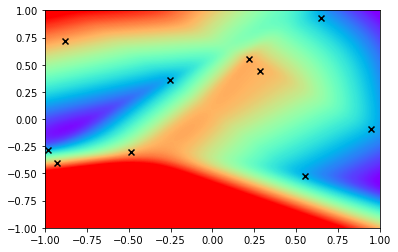} &
		\includegraphics[width=0.31\textwidth]{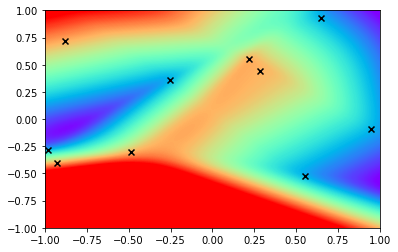} &
		\includegraphics[width=0.045\textwidth]{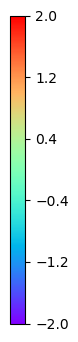}\\
		$\alpha=10^{1}$  &  $\alpha=10^{2}$ & $\alpha=\infty$ \\
		\includegraphics[width=0.31\textwidth]{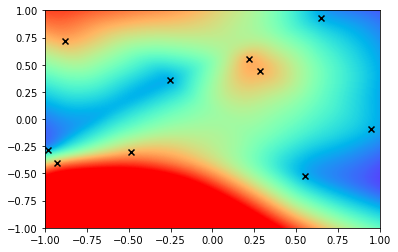} &
		\includegraphics[width=0.31\textwidth]{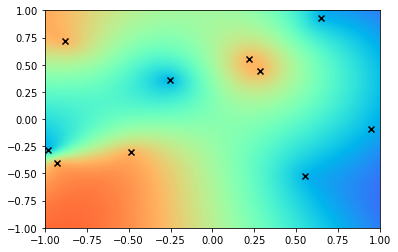} & 
		\includegraphics[width=0.31\textwidth]{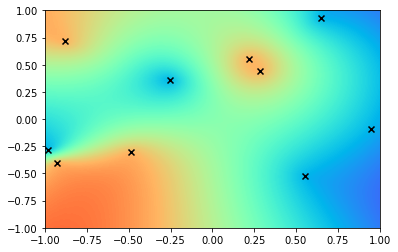} &
		\includegraphics[width=0.045\textwidth]{Images/colorbar}
	\end{tabular}
	\caption{Solutions of~\eqref{eq:TraingDiscrete} for several values of $\alpha$. The range is clipped to $[-2,2]$.}
	\label{fig:interpolation}
\end{figure}

\begin{remark}
	In principal, \eqref{eq:Interp_prob} is still over-parameterized, even in the form \eqref{eq:equiv_constr}.
	Essentially, it suffices to consider $\mathcal M^+(\{\pm 1/\sqrt 2\} \times \S^2/\sqrt 2) \subset \mathcal M^+(\S^3)$ in~\eqref{eq:equiv_constr} to realize any NN.
	This has the advantage that we only need to optimize over two 2D measures instead of a 3D one, which considerably reduces the computation time.
	Unfortunately, there is no theoretical guarantee that the optimal measures $\hat{\boldsymbol \mu}^*_\alpha$ must be supported on $\vec P$.
	However, we observed numerically that the assumption that $\supp(\hat{\boldsymbol \mu}^*_\alpha) \subset P$ leads essentially to the same results (Figure~\ref{fig:interpolation} and Figure~\ref{fig:pathcomparison}).
	Therefore, we propose to replace $\mathcal{M}^+ (\vec Q)$ by $\mathcal{M}^+ (\vec P)$ in \eqref{eq:TraingDiscrete} to decrease the computational cost.
\end{remark}

\subsection{Dynamic Viewpoint Based on Gradient Descent}\label{sec:GradDynamics}
Next, we illustrate the implicit regularizing effect of gradient descent training for the loss
\begin{equation}\label{eq:SquaredLoss}
	\sum_{k=1}^{10} \biggl \vert \frac{1}{2\cdot250^2}\sum_{l=1}^{2\cdot250^2} \beta^2  \phi_{\mathrm{ReLU}}(\vec w_l,\vec x_k) - y_k \biggr \vert^2,
\end{equation}
with $\vec w_l = (w_{1,l},\vec w_{2,l}) \in \R^4$.
To make a link with our approach in Section~\ref{sec:VarProb}, the $\vec w_l$ are initialized as the points from $\vec P$.
Depending on the initialization scale $\beta$ in~\eqref{eq:SquaredLoss}, gradient-descent training leads to very different results, as discussed in~\cite{chizat2019lazy, woodworth2020kernel}.
For all parameters $\beta$, we have chosen a sufficiently small stepsize and iterated gradient descent until convergence.
The obtained empirical measure corresponding to the scale $\beta$ is denoted by $\hat{\boldsymbol \nu}^*_\beta$.

A natural question is to investigate how the solutions induced by $\hat{\boldsymbol \nu}^*_\beta$ compare to the ones induced by $\hat{\boldsymbol \mu}^*_\alpha$.
A visual comparison  is provided in Figure~\ref{fig:pathcomparison}.
\begin{figure}[hp]
	\centering
	\begin{tabularx}{\textwidth}{Xcccc}
		\strut{} & $\alpha=\beta=\num{1e0}$ & $\alpha=\beta=\num{3.3e0}$  & $\alpha=\beta=\num{6.6e0}$ & \strut{}\\
		VP & \multicolumn{1}{m{.29\textwidth}}{\includegraphics[width=0.29\textwidth]{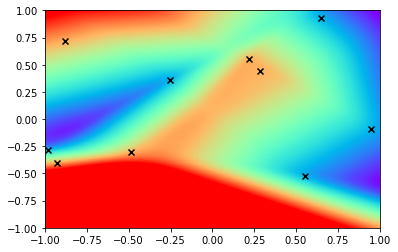}} & 
		\multicolumn{1}{m{.29\textwidth}}{\includegraphics[width=0.29\textwidth]{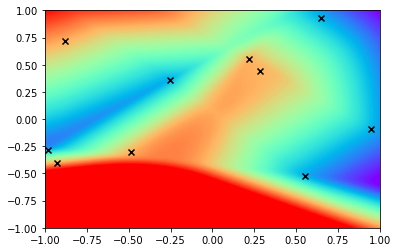}} &
		\multicolumn{1}{m{.29\textwidth}}{\includegraphics[width=0.29\textwidth]{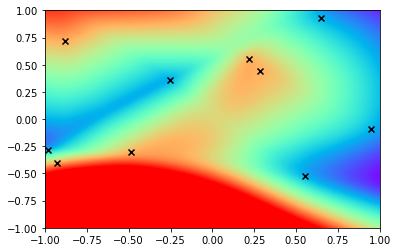}} &
		\multicolumn{1}{m{.042\textwidth}}{\includegraphics[width=0.042\textwidth]{Images/colorbar}}\\
		GD & \multicolumn{1}{m{.29\textwidth}}{\includegraphics[width=0.29\textwidth]{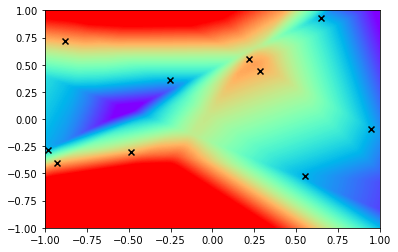}} & 
		\multicolumn{1}{m{.29\textwidth}}{\includegraphics[width=0.29\textwidth]{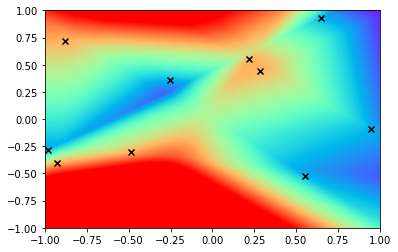}} &
		\multicolumn{1}{m{.29\textwidth}}{\includegraphics[width=0.29\textwidth]{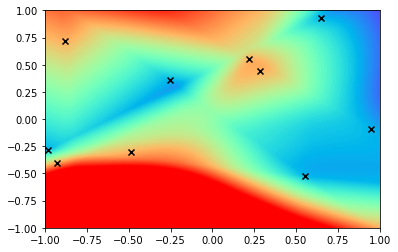}}&
		\multicolumn{1}{m{.042\textwidth}}{\includegraphics[width=0.042\textwidth]{Images/colorbar}}\\[.5ex]
		\strut{} & $\alpha=\beta=\num{1e1}$  &  $\alpha=\beta=\num{1.6e1}$   & $\alpha=\beta=\num{2.4e1}$ & \strut{}\\
		VP & \multicolumn{1}{m{.29\textwidth}}{\includegraphics[width=0.29\textwidth]{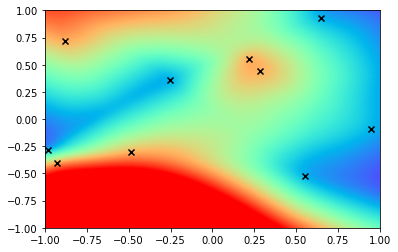}} & 
		\multicolumn{1}{m{.29\textwidth}}{\includegraphics[width=0.29\textwidth]{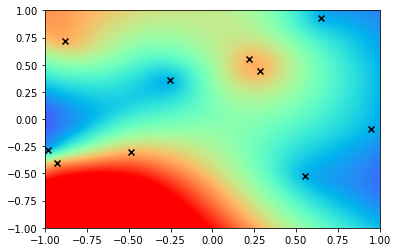}}&
		\multicolumn{1}{m{.29\textwidth}}{\includegraphics[width=0.29\textwidth]{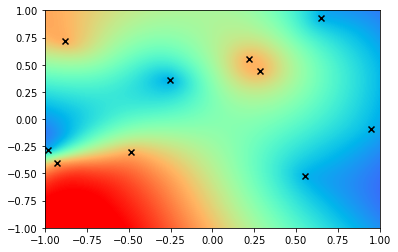}} &
		\multicolumn{1}{m{.042\textwidth}}{\includegraphics[width=0.042\textwidth]{Images/colorbar}}\\
		GD &\multicolumn{1}{m{.29\textwidth}}{\includegraphics[width=0.29\textwidth]{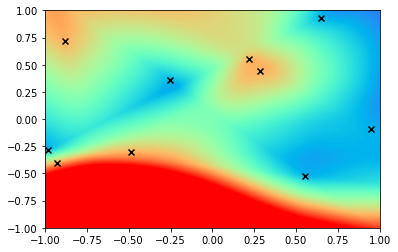}} & 
		\multicolumn{1}{m{.29\textwidth}}{\includegraphics[width=0.29\textwidth]{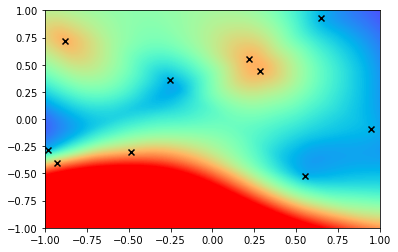}} &
		\multicolumn{1}{m{.29\textwidth}}{\includegraphics[width=0.29\textwidth]{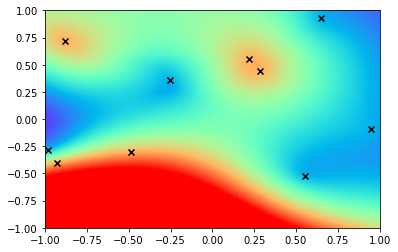}} &
		\multicolumn{1}{m{.042\textwidth}}{\includegraphics[width=0.042\textwidth]{Images/colorbar}}\\[.5ex]
		\strut{} & $\alpha=\beta=\num{4.2e1}$  &  $\alpha=\beta=\num{4.8e1}$   & $\alpha=\beta=\num{6.6e1}$ & \strut{}\\
		VP & \multicolumn{1}{m{.29\textwidth}}{\includegraphics[width=0.29\textwidth]{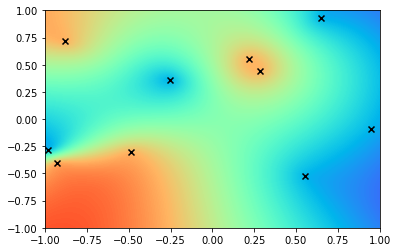}} & 
		\multicolumn{1}{m{.29\textwidth}}{\includegraphics[width=0.29\textwidth]{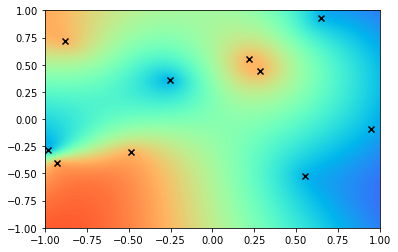}} &
		\multicolumn{1}{m{.29\textwidth}}{\includegraphics[width=0.29\textwidth]{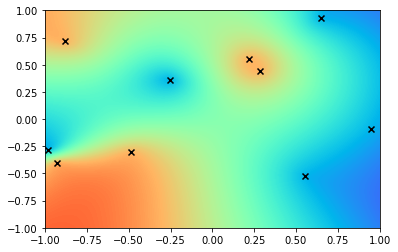}} &
		\multicolumn{1}{m{.042\textwidth}}{\includegraphics[width=0.042\textwidth]{Images/colorbar}}\\
		GD &\multicolumn{1}{m{.29\textwidth}}{\includegraphics[width=0.29\textwidth]{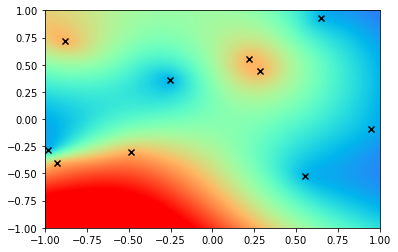}}& 
		\multicolumn{1}{m{.29\textwidth}}{\includegraphics[width=0.29\textwidth]{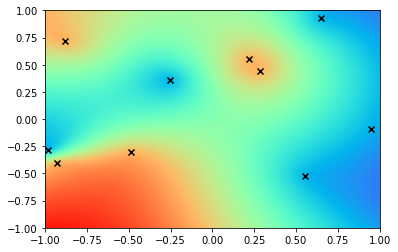}}&
		\multicolumn{1}{m{.29\textwidth}}{\includegraphics[width=0.29\textwidth]{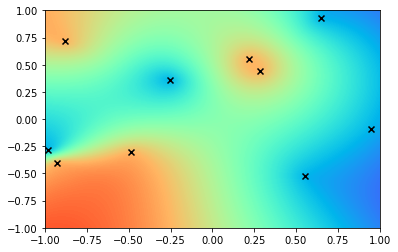}} &
		\multicolumn{1}{m{.042\textwidth}}{\includegraphics[width=0.042\textwidth]{Images/colorbar}}
	\end{tabularx}
	\caption{Optimization paths for gradient-descent training (GD) and the solutions of \eqref{eq:TraingDiscrete} (VP) restricted to $\mathcal{M}^+ (\vec P)$.
		The plots are clipped to $[-2,2]$.}
	\label{fig:pathcomparison}
\end{figure}
For larger values of $\alpha$ and $\beta$, the solutions corresponding to the same values are very similar.
As predicted by our theory, the solutions induced by $\hat{\boldsymbol \mu}^*_\alpha$ indeed approach $f^*_\infty$ associated to the kernel formulation~\eqref{eq:KernelSetting} for $\alpha \to \infty$.
The same behavior was predicted for the solutions corresponding to $\hat{\boldsymbol \nu}^*_\beta$ in~\cite{chizat2019lazy}.
Although the solutions start to differ for decreasing values of $\alpha$ and $\beta$, the path itself remains similar.
The path becomes significantly different only for small values of $\alpha$ and $\beta$.
However, for increasing width, the limits $\alpha\to 0$ and $\beta\to0$ both lead to solutions of \eqref{eq:RichRegime}.
Aside from this visual analysis, we can also examine the values of \smash{$\widehat{W}_{2,\epsilon}^2(\hat{\boldsymbol \nu}^*_\beta, \alpha^2 \hat{\boldsymbol \mu}_0) - \widehat{W}_{2,\epsilon}^2(\hat{\boldsymbol \mu}^*_\alpha, \alpha^2 \hat{\boldsymbol \mu}_0)$}.
A heat map is provided in Figure~\ref{fig:energy}, and the exact values are given in Table~\ref{tab:interpolation}.
Although not necessarily contained in the optimization domain $\mathcal M^+(\vec P)$ of \eqref{eq:TraingDiscrete}, the gradient-descent-based solutions $\hat{\boldsymbol \nu}^*_\alpha$ usually have higher functional values than their variational counterparts $\hat{\boldsymbol \mu}^*_\alpha$.
Moreover, the minimal value of \smash{$\widehat{W}_{2,\epsilon}^2(\hat{\boldsymbol \nu}^*_\beta, \alpha^2 \hat{\boldsymbol \mu}_0) - \widehat{W}_{2,\epsilon}^2(\hat{\boldsymbol \mu}^*_\alpha, \alpha^2 \hat{\boldsymbol \mu}_0)$} for large and fixed $\alpha$ is always obtained for $\beta = \alpha$.
For smaller initialization scales $\alpha$, the values are very close and $\beta = \alpha$ is close to being optimal.
\begin{figure}[t]
	\centering
	\includegraphics[width=0.4\textwidth]{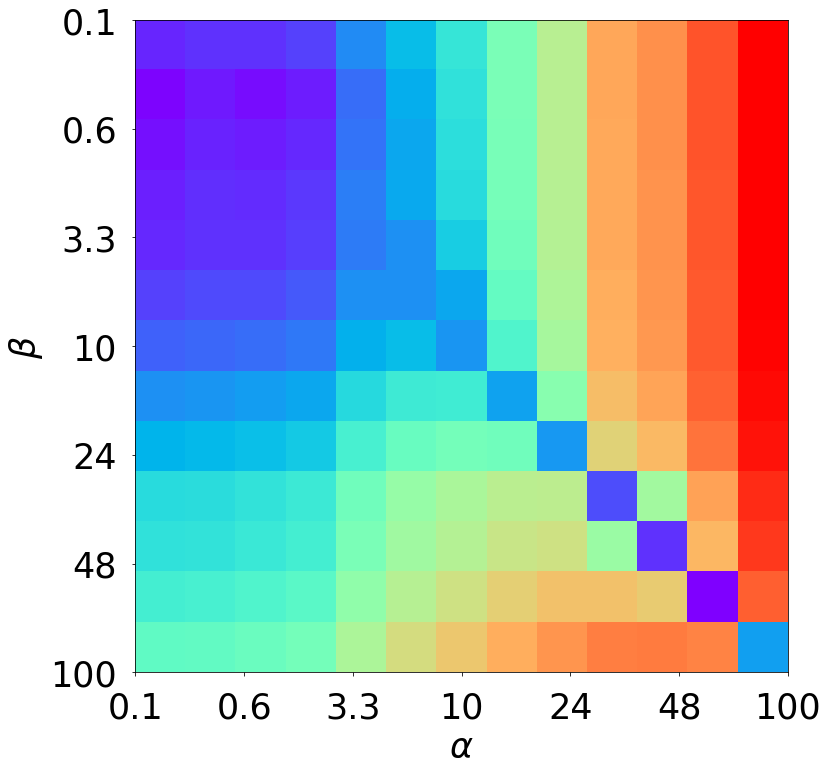} 
	\includegraphics[width=0.059\textwidth]{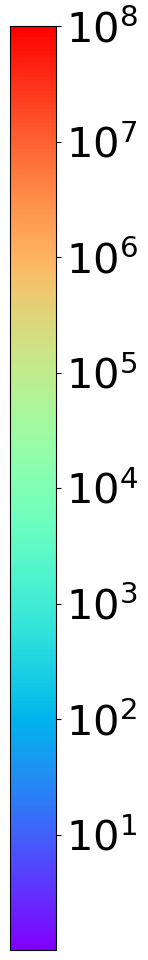} 
	\caption{Heat map of \smash{$\widehat{W}_{2,\epsilon}^2(\hat{\boldsymbol \nu}^*_\beta, \alpha^2 \hat{\boldsymbol \mu}_0) - \widehat{W}_{2,\epsilon}^2(\hat{\boldsymbol \mu}^*_\alpha, \alpha^2 \hat{\boldsymbol \mu}_0)$}.}
	\label{fig:energy}
\end{figure}
\setlength{\tabcolsep}{4.5pt}
\begin{table}
	\centering
	\begin{tabular}{c|c|c|c|c|c|c}
		\hline
		\hline
		$\alpha$ & $\num{1e-1}$ & $\num{3.3e-1}$ & $\num{6.6e-1}$ & $\num{1e0}$ & $\num{3.3e0}$ & $\num{6.6e0}$\tabularnewline  \hline
		Min.\ of~\eqref{eq:TraingDiscrete}& $\num{7.50e1}$ & $\num{7.44e1}$ & $\num{9.64e1}$ & $\num{1.26e2}$ & $\num{4.47e2}$ & $\num{7.73e2}$\tabularnewline \hline
		$\beta=\num{1e-1}$ & {\color{blue} $\mathbf{\num{9.36e1}}$} &  $\num{9.87e1}$ & $\num{1.20e2}$ & $\num{1.58e2}$ & $\num{6.19e2}$ & $\num{1.41e3}$\tabularnewline 
		$\beta=\num{3.3e-1}$ & $\num{8.44e1}$ &  {\color{blue} $\mathbf{\num{8.88e1}}$} & $\num{1.08e2}$ & $\num{1.41e2}$ & $\num{5.33e2}$ & $\num{1.17e3}$\tabularnewline 
		$\beta=\num{6.6e-1}$ & $\num{8.73e1}$ &  $\num{9.18e1}$ & {\color{blue} $\mathbf{\num{1.11e2}}$} & $\num{1.45e2}$ & $\num{5.43e2}$ & $\num{1.12e3}$\tabularnewline 
		$\beta=\num{1e0}$ & $\num{9.18e1}$ &  $\num{9.66e1}$ & $\num{1.17e2}$ & {\color{blue}$\mathbf{\num{1.53e2}}$} & $\num{5.70e2}$ & $\num{1.12e3}$\tabularnewline 
		$\beta=\num{3.3e0}$ & $\num{9.45e1}$ &  $\num{9.93e1}$ & $\num{1.20e2}$ & $\num{1.57e2}$ & {\color{blue}$\mathbf{\num{5.66e2}}$} & $\num{9.68e2}$\tabularnewline 
		$\beta=\num{6.6e0}$ & $\num{1.07e2}$ &  $\num{1.13e2}$ & $\num{1.37e2}$ & $\num{1.79e2}$ & $\num{6.43e2}$ & {\color{blue} $\mathbf{\num{9.60e2}}$}\tabularnewline 
		$\beta=\num{1e1}$ & $\num{1.39e2}$ &  $\num{1.47e2}$ & $\num{1.80e2}$ & $\num{2.37e2}$ & $\num{8.95e2}$ & $\num{1.41e3}$\tabularnewline 
		$\beta=\num{1.6e1}$ & $\num{2.70e2}$ &  $\num{2.88e2}$ & $\num{3.57e2}$ & $\num{4.77e2}$ & $\num{2.05e3}$ & $\num{4.36e3}$\tabularnewline 
		$\beta=\num{2.4e1}$ & $\num{5.82e2}$ & $\num{6.27e2}$ &  $\num{7.89e2}$ & $\num{1.06e3}$ & $\num{5.14e3}$ & $\num{1.36e4}$\tabularnewline 
		$\beta=\num{4.2e1}$ & $\num{1.77e3}$ &  $\num{1.92e3}$ & $\num{2.45e3}$ & $\num{3.36e3}$ & $\num{1.77e4}$ & $\num{5.57e4}$\tabularnewline 
		$\beta= \num{4.8e1}$& $\num{2.31e3}$ & $\num{2.51e3}$ &  $\num{3.21e3}$ & $\num{4.41e3}$ & $\num{2.37e4}$ & $\num{7.62e4}$\tabularnewline 
		$\beta=\num{6.6e1}$ & $\num{4.38e3}$ &  $\num{4.78e3}$ & $\num{6.12e3}$ & $\num{8.44e3}$ & $\num{4.66e4}$ & $\num{1.56e5}$\tabularnewline 
		\hline
		\hline
	\end{tabular}
	\vspace{.5cm}
	
	\begin{tabular}{c|c|c|c|c|c|c}
		\hline
		\hline
		$\alpha$ & $\num{1e1}$ & $\num{1.6e1}$ & $\num{2.4e1}$ & $\num{4.2e1}$ & $\num{4.8e1}$ & $\num{6.6e1}$\tabularnewline \hline
		Min.\ of~\eqref{eq:TraingDiscrete}& $\num{7.27e2}$ & $\num{4.60e2}$ & $\num{3.38e2}$ & $\num{2.88e2}$ & $\num{2.84e2}$ & $\num{2.79e2}$\tabularnewline \hline
		$\beta=\num{1e-1}$ & $\num{3.39e3}$ &  $\num{2.39e4}$ & $\num{1.63e5}$ & $\num{2.08e6}$ & $\num{3.74e6}$ & $\num{1.47e7}$\tabularnewline 
		$\beta=\num{3.3e-1}$ & $\num{3.03e3}$ &  $\num{2.38e4}$ & $\num{1.66e5}$ & $\num{2.10e6}$ & $\num{3.74e6}$ & $\num{1.48e7}$\tabularnewline 
		$\beta=\num{6.6e-1}$ & $\num{2.69e3}$ &  $\num{2.20e4}$ & $\num{1.59e5}$ & $\num{2.06e6}$ & $\num{3.71e6}$ & $\num{1.46e7}$\tabularnewline 
		$\beta=\num{1e0}$ & $\num{2.47e3}$ &  $\num{2.04e4}$ & $\num{1.52e5}$ & $\num{2.02e6}$ & $\num{3.65e6}$ & $\num{1.45e7}$\tabularnewline 
		$\beta=\num{3.3e0}$ & $\num{1.79e3}$ &  $\num{1.72e4}$ & $\num{1.40e5}$ & $\num{1.95e6}$ & $\num{3.55e6}$ & $\num{1.42e7}$\tabularnewline 
		$\beta=\num{6.6e0}$ & $\num{1.07e3}$ &  $\num{1.22e4}$ & $\num{1.20e5}$ & $\num{1.83e6}$ & $\num{3.36e6}$ & $\num{1.37e7}$\tabularnewline 
		$\beta=\num{1e1}$ & {\color{blue}$\mathbf{\num{9.50e2}}$} &  $\num{6.78e3}$ & $\num{9.27e4}$ & $\num{1.64e6}$ & $\num{3.06e6}$ & $\num{1.29e7}$\tabularnewline 
		$\beta=\num{1.6e1}$ & $\num{4.37e3}$ &  {\color{blue} $\mathbf{\num{7.66e2}}$} & $\num{3.51e4}$ & $\num{1.16e6}$ & $\num{2.30e6}$ & $\num{1.07e7}$\tabularnewline 
		$\beta=\num{2.4e1}$ & $\num{2.01e4}$ &  $\num{1.70e4}$ & {\color{blue} $\mathbf{\num{5.74e2}}$} & $\num{5.65e5}$ & $\num{1.31e6}$ & $\num{7.63e6}$\tabularnewline 
		$\beta=\num{4.2e1}$ & $\num{1.03e5}$ &  $\num{1.73e5}$ & $\num{1.86e5}$ & {\color{blue} $\mathbf{\num{3.31e2}}$} & $\num{8.28e4}$ & $\num{2.50e6}$\tabularnewline 
		$\beta=\num{4.8e1}$ & $\num{1.45e5}$ &  $\num{2.62e5}$ & $\num{3.31e5}$ & $\num{6.36e4}$ & {\color{blue} $\mathbf{\num{3.08e2}}$} & $\num{1.40e6}$\tabularnewline 
		$\beta=\num{6.6e1}$ & $\num{3.16e5}$ &  $\num{6.40e5}$ & $\num{1.01e6}$ & $\num{1.01e6}$ & $\num{7.44e5}$ & {\color{blue} $\mathbf{\num{2.87e2}}$}\tabularnewline 
		\hline
		\hline
	\end{tabular}
	\caption{Values of $(1 + \alpha^2)\widehat{W}_{2,\epsilon}^2(\hat{\boldsymbol \nu}^*_\beta,\alpha^2 \hat{\boldsymbol \mu}_0)$ in terms of the scale $\beta$ of gradient-descent training.
		The diagonal $\alpha = \beta$ is highlighted in bold.}
	\label{tab:interpolation}
\end{table}

\section{Conclusions}\label{sec:Conclusions}
In this paper, we have introduced the scaling path of a neural network.
It involves the Hellinger--Kantorovich distance (a.k.a.\ Wasserstein--Fisher--Rao distance) and depends on an initialization scale.
As main contribution, we have shown that the solutions of these paths depend continuously on the initialization scale, which makes them well-behaved objects amendable to further theoretical analyses.
The relevance of the scaling path is demonstrated by a small-scale numerical example, in which we observed that the scaling path can be indeed qualitatively related to the training dynamics of gradient descent at large times, namely, the endpoint of the optimization path.

\section*{Acknowledgment}
The research leading to these results was supported by the European Research Council (ERC) under European Union’s Horizon 2020 (H2020), Grant Agreement - Project No 101020573 FunLearn.

\bibliographystyle{abbrv}
\bibliography{references}

\begin{thebibliography}{10}

\bibitem{ali2020implicit}
A.~Ali, E.~Dobriban, and R.~Tibshirani.
\newblock The implicit regularization of stochastic gradient flow for least
  squares.
\newblock In {\em International Conference on Machine Learning}, pages
  233--244. PMLR, 2020.

\bibitem{AGS2005}
L.~Ambrosio, N.~Gigli, and G.~Savar\'e.
\newblock {\em Gradient Flows in Metric Spaces and in the Space of Probability
  Measures}.
\newblock Birkh\"auser, Basel, 2005.

\bibitem{arora2019fine}
S.~Arora, S.~Du, W.~Hu, Z.~Li, and R.~Wang.
\newblock Fine-grained analysis of optimization and generalization for
  overparameterized two-layer neural networks.
\newblock In {\em International Conference on Machine Learning}, pages
  322--332. PMLR, 2019.

\bibitem{Bach2017}
F.~Bach.
\newblock Breaking the curse of dimensionality with convex neural networks.
\newblock {\em Journal of Machine Learning Research}, 18(19):1--53, 2017.

\bibitem{Berlinet2004}
A.~Berlinet and C.~Thomas-Agnan.
\newblock {\em Reproducing Kernel {H}ilbert Spaces in Probability and
  Statistics}.
\newblock Kluwer Academic Publishers, Boston, MA, 2004.

\bibitem{BieMai2019}
A.~Bietti and J.~Mairal.
\newblock On the inductive bias of neural tangent kernels.
\newblock In {\em Advances in Neural Information Processing Systems},
  volume~32, pages 12556--12567, 2019.

\bibitem{boursier2022gradient}
E.~Boursier, L.~Pillaud-Vivien, and N.~Flammarion.
\newblock Gradient flow dynamics of shallow {ReLU} networks for square loss and
  orthogonal inputs.
\newblock In {\em Advances in Neural Information Processing Systems},
  volume~35, pages 20105--20118, 2022.

\bibitem{Braides02}
A.~Braides.
\newblock {\em {$\Gamma$}-Convergence for Beginners}.
\newblock Oxford University Press, Oxford, 2002.

\bibitem{CheVanBru2022}
Z.~Chen, E.~Vanden-Eijnden, and J.~Bruna.
\newblock A functional-space mean-field theory of partially-trained three-layer
  neural networks.
\newblock {\em ArXiv:2210.16286}, 2022.

\bibitem{CB2020}
L.~Chizat and F.~Bach.
\newblock Implicit bias of gradient descent for wide two-layer neural networks
  trained with the logistic loss.
\newblock In {\em Proceedings of Machine Learning Research}, volume 125, pages
  1305--1338. PMLR, 2020.

\bibitem{chizat2019lazy}
L.~Chizat, E.~Oyallon, and F.~Bach.
\newblock On lazy training in differentiable programming.
\newblock In {\em Advances in Neural Information Processing Systems},
  volume~32, pages 2933--2943, 2019.

\bibitem{chizat2018interpolating}
L.~Chizat, G.~Peyr{\'e}, B.~Schmitzer, and F.-X. Vialard.
\newblock An interpolating distance between optimal transport and
  {F}isher--{R}ao metrics.
\newblock {\em Foundations of Computational Mathematics}, 18:1--44, 2018.

\bibitem{CW05}
P.~L. Combettes and V.~R. Wajs.
\newblock Signal recovery by proximal forward-backward splitting.
\newblock {\em Multiscale Modeling $\&$ Simulation}, 4:1168--1200, 2005.

\bibitem{FSVATP2018}
J.~Feydy, T.~S\'ejourn\'e, F.-X. Vialard, S.~Amari, A.~Trouv\'e, and
  G.~Peyr\'e.
\newblock Interpolating between optimal transport and {MMD} using {Sinkhorn}
  divergences.
\newblock In {\em Proceedings of Machine Learning Research}, volume~89, pages
  2681--2690. PMLR, 2019.

\bibitem{GolStu2014}
T.~Goldstein, C.~Studer, and R.~Baraniuk.
\newblock A field guide to forward-backward splitting with a {FASTA}
  implementation.
\newblock {\em arXiv preprint arXiv:1411.3406}, 2014.

\bibitem{JacGab2018}
A.~Jacot, F.~Gabriel, and C.~Hongler.
\newblock Neural tangent kernel: Convergence and generalization in neural
  networks.
\newblock In {\em Advances in Neural Information Processing Systems},
  volume~31, pages 8580--8589, 2018.

\bibitem{kondratyev2016new}
S.~Kondratyev, L.~Monsaingeon, and D.~Vorotnikov.
\newblock A new optimal transport distance on the space of finite {R}adon
  measures.
\newblock {\em Advances in Differential Equations}, 21(11/12):1117--1164, 2016.

\bibitem{li2019stochastic}
Q.~Li, C.~Tai, and E.~Weinan.
\newblock Stochastic modified equations and dynamics of stochastic gradient
  algorithms {I}: Mathematical foundations.
\newblock {\em Journal of Machine Learning Research}, 20(1):1474--1520, 2019.

\bibitem{LMS2015}
M.~Liero, A.~Mielke, and G.~Savar\'{e}.
\newblock Optimal transport in competition with reaction: the
  {H}ellinger-{K}antorovich distance and geodesic curves.
\newblock {\em SIAM Journal on Mathematical Analysis}, 48(4):2869--2911, 2016.

\bibitem{LMS2018}
M.~Liero, A.~Mielke, and G.~Savar\'{e}.
\newblock Optimal entropy-transport problems and a new
  {H}ellinger-{K}antorovich distance between positive measures.
\newblock {\em Inventiones Mathematicae}, 211(3):969--1117, 2018.

\bibitem{lyu2021gradient}
K.~Lyu, Z.~Li, R.~Wang, and S.~Arora.
\newblock Gradient descent on two-layer nets: Margin maximization and
  simplicity bias.
\newblock In {\em Advances in Neural Information Processing Systems},
  volume~34, pages 12978--12991, 2021.

\bibitem{NS20}
S.~Neumayer and G.~Steidl.
\newblock From optimal transport to discrepancy.
\newblock In {\em Handbook of Mathematical Models and Algorithms in Computer
  Vision and Imaging}. Springer, 2021.

\bibitem{NeuUns2022}
S.~Neumayer and M.~Unser.
\newblock Explicit representations for {B}anach subspaces of {L}izorkin
  distributions.
\newblock {\em Analysis and Applications}, 2023.

\bibitem{Ongie2020b}
G.~Ongie, R.~Willett, D.~Soudry, and N.~Srebro.
\newblock A function space view of bounded norm infinite width {ReLU} nets: The
  multivariate case.
\newblock In {\em International Conference on Learning Representations}, 2020.

\bibitem{pesme2021implicit}
S.~Pesme, L.~Pillaud-Vivien, and N.~Flammarion.
\newblock Implicit bias of {SGD} for diagonal linear networks: a provable
  benefit of stochasticity.
\newblock In {\em Advances in Neural Information Processing Systems},
  volume~34, pages 29218--29230, 2021.

\bibitem{razin2020implicit}
N.~Razin and N.~Cohen.
\newblock Implicit regularization in deep learning may not be explainable by
  norms.
\newblock In {\em Advances in Neural Information Processing Systems},
  volume~33, pages 21174--21187, 2020.

\bibitem{su2014differential}
W.~Su, S.~Boyd, and E.~Cand\'es.
\newblock A differential equation for modeling {N}esterov’s accelerated
  gradient method: {T}heory and insights.
\newblock In {\em Advances in Neural Information Processing Systems},
  volume~27, pages 2510--2518, 2014.

\bibitem{suggala2018connecting}
A.~Suggala, A.~Prasad, and P.~K. Ravikumar.
\newblock Connecting optimization and regularization paths.
\newblock In {\em Advances in Neural Information Processing Systems},
  volume~31, pages 10631--10641, 2018.

\bibitem{SwiJam06}
R.~Swinbank and R.~James~Purser.
\newblock Fibonacci grids: A novel approach to global modelling.
\newblock {\em Quarterly Journal of the Royal Meteorological Society},
  132(619):1769--1793, 2006.

\bibitem{Villani2009}
C.~Villani.
\newblock {\em Optimal Transport: Old and New}, volume 338 of {\em Grundlehren
  der Ma\-the\-ma\-tischen Wissenschaften}.
\newblock Springer-Verlag, Berlin, 2009.

\bibitem{wang2022large}
Y.~Wang, M.~Chen, T.~Zhao, and M.~Tao.
\newblock Large learning rate tames homogeneity: Convergence and balancing
  effect.
\newblock In {\em International Conference on Learning Representations}, 2022.

\bibitem{woodworth2020kernel}
B.~Woodworth, S.~Gunasekar, J.~D. Lee, E.~Moroshko, P.~Savarese, I.~Golan,
  D.~Soudry, and N.~Srebro.
\newblock Kernel and rich regimes in overparametrized models.
\newblock In {\em Proceedings of Machine Learning Research}, volume 135, pages
  3635--3673. PMLR, 2020.

\end{thebibliography}
\end{document}